\documentclass[11pt]{article}

\usepackage{amsfonts}
\usepackage{amssymb,amsmath,latexsym}
\usepackage{epsfig,amsthm}
\usepackage{subfigure}
\usepackage{fullpage}
\usepackage{url}
\usepackage{tensor}   %
\usepackage[usenames,dvipsnames]{color}
\usepackage[round]{natbib}
\usepackage[ruled,noend,noline,nofillcomment,linesnumbered]{algorithm2e}

\newcommand{\mc}{\mathcal}
\newcommand{\mbf}{\mathbf}

\newcommand{\ouralgorithm}{\textsf{PPCA}}
\newcommand{\sulq}{\textsf{SULQ}}
\newcommand{\modsulq}{\textsf{MOD-SULQ}}

\newtheorem{theorem}{Theorem}
\newtheorem{definition}{Definition}
\newtheorem{lemma}[theorem]{Lemma}

\newcommand{\priveps}{\epsilon_p}

\def\qnorm{q_{\mathrm{F}}}
\def\qangle{q_{\mathrm{A}}}
\newcommand{\frobnorm}[1]{\norm{ #1 }_{\mathrm{F}}}
\def\tr{\mathop{\rm tr}\nolimits}       %

\newcommand{\kl}{\mathbf{KL}}
\newcommand{\kldiv}[2]{\kl\left( #1 \| #2 \right)}
\newcommand{\ip}[2]{ \langle #1, #2 \rangle }
\newcommand{\norm}[1]{\left\| #1 \right\|}

\newcommand{\R}{\mathbb{R}}
\newcommand{\Sp}{\mathbb{S}}

\newcommand{\expe}[1]{\mathbb{E}\left[ #1 \right]}
\newcommand{\prob}[1]{\mathbb{P}\left( #1 \right)}
\newcommand{\E}{\mathbb{E}}
\newcommand{\pr}{\mathbb{P}}

\newcommand{\calA}{\mathcal{A}}

\renewcommand\footnotemark{}

\title{Near-Optimal Algorithms for Differentially-Private Principal Components%
\thanks{KC is with the Department of Computer Science and Engineering, University of California, San Diego, \texttt{kchaudhuri@ucsd.edu}.
ADS is with the Toyota Technological Institute at Chicago, \texttt{asarwate@ttic.edu}.
KS is with the Department of Computer Science and Engineering, University of California, San Diego, \texttt{ksinha@cs.ucsd.edu}.
KC and KS would like to thank NIH for research support under U54-HL108460.
The experimental results were made possible by support from the UCSD
FWGrid Project, NSF Research Infrastructure Grant Number EIA-0303622.
ADS was supported in part by the California Institute for Telecommunications and
Information Technology (CALIT2) at UC San Diego.}}
\author{Kamalika Chaudhuri
\and
Anand D. Sarwate
\and
Kaushik Sinha}
\date{\today}

\begin{document}

\maketitle

\begin{abstract}
Principal components analysis (PCA) is a standard tool for identifying good low-dimensional approximations to data in high dimension.  Many data sets of interest contain private or sensitive information about individuals.  Algorithms which operate on such data should be sensitive to the privacy risks in publishing their outputs.  Differential privacy is a framework for developing tradeoffs between privacy and the utility of these outputs.  In this paper we investigate the theory and empirical performance of differentially private approximations to PCA and propose a new method which explicitly optimizes the utility of the output.   
We show that the sample complexity of the proposed method differs from the existing procedure in the scaling with the data dimension, and that our method is nearly optimal in terms of this scaling.
We furthermore illustrate our results, showing that on real data there is a large performance gap between the existing method and our method.
\end{abstract}

\section{Introduction}

Dimensionality reduction is a fundamental tool for understanding complex data sets that arise in contemporary machine learning and data mining applications.  Even though a single data point can be represented by hundreds or even thousands of features, the phenomena of interest are often intrinsically low-dimensional.  By reducing the ``extrinsic'' dimension of the data to its ``intrinsic'' dimension, analysts can discover important structural relationships between features, more efficiently use the transformed data for learning tasks such as classification or regression, and greatly reduce the space required to store the data.  One of the oldest and most classical methods for dimensionality reduction is principal components analysis (PCA), which computes a low-rank approximation to the second moment matrix $A$ of a set of points in $\mathbb{R}^d$.  The rank $k$ of the approximation is chosen to be the intrinsic dimension of the data.  We view this procedure as specifying a $k$-dimensional subspace of $\mathbb{R}^d$.

Much of today's machine-learning is performed on the vast amounts of personal information collected by private companies and government agencies about individuals: examples include user or customer behaviors, demographic surveys, and test results from experimental subjects or patients. These datasets contain sensitive information about individuals and typically involve a large number of features.  It is therefore important to design machine-learning algorithms which discover important structural relationships in the data while taking into account its sensitive nature.  We study approximations to PCA which guarantee differential privacy, a cryptographically motivated definition of privacy~\citep{DworkMNS:06sensitivity} that has gained significant attention over the past few years in the machine-learning and data-mining communities~\citep{MKAGV08,MM09,M09,FS10,MCFY11}.  Differential privacy measures privacy risk by a parameter $\priveps$ that bounds the log-likelihood ratio of output of a (private) algorithm under two databases differing in a single individual.

There are many general tools for providing differential privacy.  The sensitivity method due to \cite{DworkMNS:06sensitivity} computes the desired algorithm (in our case, PCA) on the data and then adds noise  proportional to the maximum change than can be induced by changing a single point in the data set.  The PCA algorithm is very sensitive in this sense because the top eigenvector can change by $90^{\circ}$ by changing one point in the data set.  Relaxations such as smoothed sensitivity~\citep{NRS07} are difficult to compute in this setting as well.  The SUb Linear Queries (\sulq) method of~\cite{BDMN05} adds noise to the second moment matrix and then runs PCA on the noisy matrix.  As our experiments show, the noise level required by \sulq\ may severely impact the quality of approximation, making it impractical for data sets of moderate size.

The goal of this paper is to characterize the problem of differentially private PCA.  We assume that the algorithm is given $n$ data points and a target dimension $k$ and must produce a $k$-dimensional subspace that approximates that produced by the standard PCA problem.  We propose a new algorithm, \ouralgorithm, which is an instance of the exponential mechanism of~\cite{MT07}.  Unlike \sulq, \ouralgorithm\ explicitly takes into account the quality of approximation---it outputs a $k$-dimensional subspace which is biased towards subspaces close to the output of PCA.  In our case, the method corresponds to sampling from the matrix Bingham distribution.  We implement \ouralgorithm\ using a Markov Chain Monte Carlo (MCMC) procedure due to~\cite{Hoff:09bingham}; simulations show that the subspace produced by \ouralgorithm\ captures more of the variance of $A$ than \sulq.  When the MCMC procedure converges, the algorithm provides differential privacy.

In order to understand the performance gap, we prove sample complexity bounds for the case of $k = 1$ for \sulq\ and \ouralgorithm, as well as a general lower bound on the sample complexity for any differentially private algorithm.   We show that the sample complexity scales as $\Omega( d^{3/2} \sqrt{\log d} )$ for \sulq\ and as $O( d )$ for \ouralgorithm.  Furthermore, we show that any differentially private algorithm requires $\Omega( d )$ samples.  Therefore \ouralgorithm\ is nearly optimal in terms of sample complexity as a function of data dimension.  These theoretical results suggest that our experiments demonstrate the limit of how well $\priveps$-differentially private algorithms can perform, and our experiments show that this gap should persist for general $k$.  The result seems pessimistic for many applications, because the sample complexity depends on the extrinsic dimension $d$ rather than the intrinsic dimension $k$.  However, we believe this is a consequence of the fact that we make minimal assumptions on the data; our results imply that, absent additional limitations on the data set, the sample complexity differentially private PCA must grow linearly with the data dimension.

There are several interesting open questions suggested by this work.  One set of issues is computational.  Differentially privacy is a mathematical definition, but algorithms must be implemented using finite precision machines.  Privacy and computation interact in many places, including pseudorandomness, numerical stability, optimization, and in the MCMC procedure we use to implement \ouralgorithm; investigating the impact of approximate sampling is an avenue for future work.  A second set of issues is theoretical---while the privacy guarantees of \ouralgorithm\ hold for all $k$, our theoretical analysis of sample complexity applies only to $k = 1$ in which the distance and angles between vectors are related.  An interesting direction is to develop theoretical bounds for general $k$; challenges here are providing the right notion of approximation of PCA, and extending the theory using packings of Grassmann or Stiefel manifolds.  Finally, in this work we assume $k$ is given to the algorithm, but in many applications $k$ is chosen after looking at the data.  Under differential privacy, the selection of $k$ itself must be done in a differentially private manner.

\subsection*{Related Work}

Differential privacy was first proposed by~\citet{DworkMNS:06sensitivity}.  There has been an extensive literature following this work in the computer science theory, machine learning, and databases communities.  A survey of some of the theoretical work can be found in the survey by~\citet{dwork2010differential}.  Differential privacy has been shown to have strong \textit{semantic} guarantees~\citep{DworkMNS:06sensitivity,KS08} and is resistant to many attacks~\citep{GKS08} that succeed against alternative definitions of privacy.  In particular, so-called syntactic definitions of privacy \citep{S02,MGKV06,LiLV10} may be susceptible to attacks based on side-information about individuals in the database.  

There are several general approaches to constructing differentially private approximations to some desired algorithm or computation. Input perturbation~\citep{BDMN05} adds noise to the data prior to performing the desired computation, whereas output perturbation~\citep{DworkMNS:06sensitivity} adds noise to the output of the desired computation.  The exponential mechanism~\citep{MT07} can be used to perform differentially private selection based on a score function that measures the quality of different outputs.  Objective perturbation~\citep{ChaudhuriMS:11erm} adds noise to the objective function for algorithms which are convex optimizations.  These approaches and related ideas such as~\citep{NRS07, DL09} have been used to approximate a variety of statistical, machine learning, and data mining tasks under differential privacy~\citep{BCDKMT07, WassermanZ:10framework,Smith:2011:PSE:1993636.1993743, MM09, WM10,ChaudhuriMS:11erm,RubensteinBHT,NRS07,BLR08,MT07,FS10,HR12}.

This paper deals with the problem of differentially private approximations to PCA.  Prior to our work, the only proposed method for PCA was the Sub-Linear Queries (\sulq) method of~\citet{BDMN05}.  This approach adds noise to the second moment matrix of the data before calculating the singular value decomposition.  By contrast, our algorithm, \ouralgorithm, uses the exponential mechanism~\citep{MT07} to choose a $k$-dimensional subspace biased toward those which capture more of ``energy'' of the matrix.  Subsequent to our work,~\citet{KT13} have proposed a dynamic programming algorithm for differentially private low rank matrix approximation which involves sampling from a distribution induced by the exponential mechanism. The running time of their algorithm is $O(d^6)$, where $d$ is the data dimension, and it is unclear how this may affect its implementation.  Hardt and Roth \citep{HR12,HardtR:13} have studied low-rank matrix approximation under additional incoherence assumptions on the data.  In particular, \citet{HR12} consider the problem of differentially-private low-rank matrix reconstruction for applications to sparse matrices; provided certain coherence conditions hold, they provide an algorithm for constructing a rank $2k$ approximation $B$ to a matrix $A$ such that $\| A - B\|_{\mathrm{F}}$ is $O(\|A - A_k\|)$ plus some additional terms which depend on $d$, $k$ and $n$; here $A_k$ is the best rank $k$ approximation to $A$. \citet{HardtR:13} show a method for guaranteeing differential privacy under an \textit{entry-wise} neighborhood condition using the power method for calculating singular values.  Because of the additional assumptions and different model, these results are generally incomparable to ours.

In addition to these works, other researchers have examined the interplay between projections and differential privacy.  \citet{ZhouLW:09compression} analyze a differentially private data release scheme where a random linear transformation is applied to data to preserve differential privacy, and then measures how much this transformation affects the utility of a PCA of the data.  One example of a random linear transformation is random projection, popularized by the Johnson-Lindenstrauss (JL) transform.  \citet{BlockiBDS:12} show that the JL transform of the data preserves differential privacy provided the minimum singular value of the data matrix is large.  \citet{kenthapadi2012privacy} study the problem of estimating the distance between data points with differential privacy using a random projection of the data points.

There has been significant work on other notions of privacy based on manipulating entries within the database~\citep{S02,MGKV06,LiLV10}, for example by reducing the resolution of certain features to create ambiguities. For more details on these and other alternative notions of privacy see~\citet{FungSurvey} for a survey with more references.  An alternative line of privacy-preserving data-mining work~\citep{ZM07} is in the Secure Multiparty Computation setting;  %
one work~\citep{4812517} studies privacy-preserving singular value decomposition in this model. Finally, dimension reduction through random projection has been considered as a technique for sanitizing data prior to publication~\citep{Liu06randomprojection-based}; our work differs from this line of work in that we offer differential privacy guarantees, and we only release the PCA subspace, not actual data.

\section{Preliminaries}

The data given to our algorithm is a set of $n$ vectors $\mc{D} = \{x_1, x_2,\ldots, x_n\}$ where each $x_i$ corresponds to the private value of one individual, $x_i \in \R^d$, and $\|x_i\|\le 1$ for all $i$. Let $X=[x_1,\ldots, x_n]$ be the matrix whose columns are the data vectors $\{x_i\}$. Let $A=\frac{1}{n}XX^T$ denote the $d\times d$ second moment matrix of the data. The matrix $A$ is positive semidefinite, and has Frobenius norm $\norm{ A }_F$ at most $1$.

The problem of dimensionality reduction is to find a ``good'' low-rank approximation to $A$.  A popular solution is to compute a rank-$k$ matrix $\hat{A}$ which minimizes the norm $\|A - \hat{A}\|_\mathrm{F}$, where $k$ is much lower than the data dimension $d$.  The Schmidt approximation theorem~\citep{Stewart:93svdhistory} shows that the minimizer is given by the singular value decomposition, also known as the PCA algorithm in some areas of computer science.  

\begin{definition}
Suppose $A$ is a positive semidefinite matrix whose first $k$ eigenvalues are distinct.  Let the eigenvalues of $A$ be $\lambda_1(A) \ge \lambda_2(A) \ge \cdots \ge \lambda_d(A) \ge 0$ and let $\Lambda$ be a diagonal matrix with $\Lambda_{ii} = \lambda_i(A)$.  The matrix $A$ decomposes as
	\begin{align}
	A = V \Lambda V^T,
	\label{eq:pca}
	\end{align}
where $V$ is an orthonormal matrix of eigenvectors.  The top-$k$ PCA subspace of $A$ is the matrix
	\begin{align}
	V_k(A) = \left[ v_1 \ v_2 \ \cdots \  v_k \right],
	\label{eq:topk}
	\end{align}
where $v_i$ is the $i$-th column of $V$ in \eqref{eq:pca}.  The $k$-th eigengap is $\Delta_k = \lambda_k - \lambda_{k+1}$.
\end{definition}

Given the top-$k$ subspace and the eigenvalue matrix $\Lambda$, we can form an approximation $A^{(k)} = V_k(A) \Lambda_k V_k(A)^T$ to $A$, where $\Lambda_k$ contains the $k$ largest eigenvalues in $\Lambda$.
In the special case $k = 1$ we have $A^{(1)} = \lambda_1(A) v_1 v_1^T$, where $v_1$ is the eigenvector corresponding to $\lambda_1(A)$.  We refer to $v_1$ as the \textit{top eigenvector} of the data, and $\Delta = \Delta_1$ is the eigengap.  For a $d \times k$ matrix $\hat{V}$ with orthonormal columns, the quality of $\hat{V}$ in approximating $V_k(A)$ can be measured by
	\begin{align}
	\qnorm(\hat{V}) = \tr\left( \hat{V}^T A \hat{V} \right). 
	\label{eq:util:energy}
	\end{align}
The $\hat{V}$ which maximizes $q(\hat{V})$ has columns equal to $\{v_i : i \in [k]\}$, corresponding to the top-$k$ eigenvectors of $A$.

Our theoretical results on the utility of our PCA approximation apply to the special case $k = 1$.  We prove results about the inner product between the output vector $\hat{v}_1$ and the true top eigenvector $v_1$:
	\begin{align}
	\qangle( \hat{v}_1 ) = \left| \ip{ \hat{v}_1 }{ v_1 } \right|.
	\label{eq:util:dist}
	\end{align}
The utility in \eqref{eq:util:dist} is related to \eqref{eq:util:energy}.  If we write $\hat{v}_1$ in the basis spanned by $\{v_i\}$, then
	\begin{align*}
	\qnorm( \hat{v}_1 ) = \lambda_1 \qangle( \hat{v}_1 )^2 + \sum_{i=2}^{d} \lambda_i \ip{ \hat{v}_1 }{ v_i }^2.
	\end{align*}
Our proof techniques use the geometric properties of $\qangle(\cdot)$.

\begin{definition}
A randomized algorithm $\mc{A}(\cdot)$ is an $(\rho,\eta)$-close approximation to the top eigenvector if for all data sets $\mc{D}$ of $n$ points we have
	\begin{align*}
	\prob{ \qangle( \mc{A}( \mc{D} ) ) \ge \rho } \ge 1 - \eta,
	\end{align*}
where the probability is taken over $\mc{A}(\cdot)$.
\end{definition}

We study approximations to $\mc{A}$ to PCA that preserve the privacy of the underlying data.  The notion of privacy that we use is differential privacy, which quantifies the privacy guaranteed by a randomized algorithm $\mc{A}$ applied to a data set $\mathcal{D}$.  %

\begin{definition}
An algorithm $\mc{A}(\mc{B})$ taking values in a set $\mc{T}$ provides $\priveps$-differential privacy if
\begin{align*}
\sup_{\mc{S}} \sup_{\mc{D}, \mc{D}'} \frac{ \mu\left( \mc{S} ~|~ \mc{B} = \mc{D} \right)
	}{
\mu\left( \mc{S} ~|~ \mc{B} = \mc{D}' \right)
	}	
	\le
	e^{\priveps},
\end{align*}
where the first supremum is over all measurable $\mc{S} \subseteq \mc{T}$, the second is
over all data sets $\mc{D}$ and $\mc{D}'$ differing in a single entry, and
$\mu(\cdot|\mc{B})$ is the conditional distribution (measure) on $\mc{T}$
induced by the output $\mc{A}(\mc{B})$ given a data set $\mc{B}$.  The ratio is interpreted to be 1 whenever the numerator and denominator are both 0.
\label{def:densitypriv}
\end{definition}

\begin{definition}
An algorithm $\mc{A}(\mc{B})$ taking values in a set $\mc{T}$ provides $(\priveps,\delta)$-differential privacy if
\begin{align*}
\prob{ \mc{A}( \mc{D} ) \in \mc{S} }
	\le
	e^{\priveps} \prob{ \mc{A}( \mc{D}' ) \in \mc{S} } + \delta,
\end{align*}
for all measurable $\mc{S} \subseteq \mc{T}$ and all data sets $\mc{D}$ and $\mc{D}'$ 
differing in a single entry.
\end{definition}

Here $\priveps$ and $\delta$ are privacy parameters, where low $\priveps$ and $\delta$ ensure more privacy~\citep{DworkMNS:06sensitivity,WassermanZ:10framework,DworkKMMN:06ourselves}.  The second privacy guarantee is weaker; the parameter $\delta$ bounds the probability of failure, and is typically chosen to be quite small.  In our experiments we chose small but constant $\delta$---\citet{GKS08} suggest $\delta < \frac{1}{n^2}$ is more appropriate. 

In this paper we are interested in proving results on the sample complexity of differentially private algorithms that approximate PCA.  That is, for a given $\priveps$ and $\rho$, how large must the number of individuals $n$ in the data set be such that the algorithm is both $\priveps$-differentially private and a $(\rho,\eta)$-close approximation to PCA?  It is well known that as the number of individuals $n$ grows, it is easier to guarantee the same level of privacy with relatively less noise or perturbation, and therefore the utility of the approximation also improves.  Our results characterize how the privacy $\priveps$ and utility $\rho$ scale with $n$ and the tradeoff between them for fixed $n$.  We show that the sample complexity depends on the eigengap $\Delta$.

\section{Algorithms and results}

In this section we describe differentially private techniques for approximating \eqref{eq:topk}.  The first is a modified version of the Sub-Linear Queries (\sulq) method~\citep{BDMN05}.  Our new algorithm for differentially-private PCA, \ouralgorithm, is an instantiation of the exponential mechanism due to~\cite{MT07}.  Both procedures are differentially private approximations to the top-$k$ subspace: \sulq\ guarantees $(\priveps,\delta)$-differential privacy and \ouralgorithm\ guarantees $\priveps$-differential privacy.

\subsection{ Input perturbation}

The only differentially-private approximation to PCA prior to this work is the \sulq\ method~\citep{BDMN05}. The \sulq\ method perturbs each entry of the empirical second moment matrix $A$ to ensure differential privacy and releases the top-$k$ eigenvectors of this perturbed matrix.  More specifically, \sulq\ recommends adding a matrix $N$ of i.i.d. Gaussian noise of variance $\frac{8 d^2 \log^2(d/\delta)}{n^2 \priveps^2}$ and applies the PCA algorithm to $A + N$.  This guarantees a weaker privacy definition known as $(\priveps,\delta)$-differential privacy.  One problem with this approach is that with probability 1 the matrix $A+N$ is not symmetric, so the largest eigenvalue may not be real and the entries of the corresponding eigenvector may be complex.  Thus the \textsf{SULQ} algorithm, as written, is not a good candidate for approximating PCA.  

It is easy to modify \sulq\ to produce a an eigenvector with real entries that guarantees $(\priveps,\delta)$ differential privacy.  In Algorithm \ref{alg:modified_sulq}, instead of adding an asymmetric Gaussian matrix, we add  a symmetric matrix with i.i.d. Gaussian entries $N$.  That is, for $1 \le i \le j \le d$, the variable $N_{ij}$ is an independent Gaussian random variable with variance $\beta^2$.  Note that this matrix is symmetric but not necessarily positive semidefinite, so some eigenvalues may be negative but the eigenvectors are all real.  A derivation for the noise variance in \eqref{eq:sulqsd} of Algorithm \ref{alg:modified_sulq} is given in Theorem \ref{thm:sulq:priv}.  An alternative is to add Laplace noise of an appropriate variance to each entry---this would guarantee $\priveps$-differential privacy.

\begin{algorithm*}
\SetKwInput{Input}{inputs}
\SetKwInput{Output}{outputs}
\caption{Algorithm \modsulq\ (input pertubation)}
\label{alg:modified_sulq}
\Input{$d\times n$ data matrix $X$, privacy parameter $\priveps$, parameter $\delta$}
\Output{ $d \times k$ matrix $\hat{V}_k = [\hat{v}_1\ \hat{v}_2\ \cdots \ \hat{v}_k ]$ with orthonormal columns }
Set $A=\frac{1}{n}XX^T$.\;
Set 
	\begin{align}
	\beta = \frac{d+1}{ n \priveps} \sqrt{ 2 \log \left( \frac{d^2 + d}{ \delta 2 \sqrt{2 \pi} } \right) } + \frac{1}{ n \sqrt{\priveps}}.
	\label{eq:sulqsd}
	\end{align}
Generate a $d\times d$ symmetric random matrix $N$ whose entries are i.i.d. drawn from $\mathcal{N}\left(0, \beta^2\right)$. \; 
Compute $\hat{V}_k = V_k(A+N)$ according to \eqref{eq:topk}. \;
\end{algorithm*}

\subsection{Exponential mechanism}

Our new method, \ouralgorithm, randomly samples a $k$-dimensional subspace from a distribution that ensures differential privacy and is biased towards high utility.  The distribution from which our released subspace is sampled is known in the statistics literature as the matrix Bingham distribution~\citep{Chikuse:03special},
which we denote by $\mathsf{BMF}_k(B)$.  The algorithm and its privacy properties apply to general $k < d$ but our theoretical results on the utility focus on the special case $k = 1$.  The matrix Bingham distribution takes values on the set of all $k$-dimensional subspaces of $\mathbb{R}^d$ and has a density equal to 
	\begin{align}
	f(V) =  \frac{1}{ \tensor*[_1]{F}{_1}\left(\frac{1}{2} k, \frac{1}{2} d, B\right) } \exp( \tr( V^T B V ) ),
	\label{eq:bmf:density}
	\end{align}
where $V$ is a $d \times k$ matrix whose columns are orthonormal and $\tensor*[_1]{F}{_1}\left(\frac{1}{2} k, \frac{1}{2} d, B\right)$ is a confluent hypergeometric function~\citep[p.33]{Chikuse:03special}.  %

\begin{algorithm*}
\SetKwInput{Input}{inputs}
\SetKwInput{Output}{outputs}
\caption{Algorithm \ouralgorithm\ (exponential mechanism)}
\label{alg:exp_mech}
\Input{ $d\times n$ data matrix $X$, privacy parameter $\priveps$, dimension $k$ }
\Output{ $d \times k$ matrix $\hat{V}_k = [\hat{v}_1\ \hat{v}_2\ \cdots \ \hat{v}_k ]$ with orthonormal columns }
Set $A=\frac{1}{n}XX^T$ \;
Sample $\hat{V}_k = \mathsf{BMF}\left(n \frac{\priveps}{2} A \right)$ \;
\end{algorithm*}

By combining results on the exponential mechanism along with properties of PCA algorithm, we can show that this procedure is differentially private.   In many cases, sampling from the distribution specified by the exponential mechanism may be expensive computationally, especially for continuous-valued outputs.  We implement \ouralgorithm\ using a recently-proposed Gibbs sampler due to~\cite{Hoff:09bingham}.  Gibbs sampling is a popular Markov Chain Monte Carlo (MCMC) technique in which samples are generated according to a Markov chain whose stationary distribution is the density in \eqref{eq:bmf:density}.  Assessing the ``burn-in time'' and other factors for this procedure is an interesting question in its own right; further details are in Section \ref{sec:gibbs}.

\subsection{Other approaches}

There are other general algorithmic strategies for guaranteeing differential privacy.  The sensitivity method~\citep{DworkMNS:06sensitivity} adds noise proportional to the maximum change that can be induced by changing a single point in the data set.  Consider a data set $\mc{D}$ with $m+1$ copies of a unit vector $u$ and $m$ copies of a unit vector $u'$ with $u \perp u'$ and let $\mc{D}'$ have $m$ copies of $u$ and $m+1$ copies of $u'$.  Then $v_1(\mc{D}) = u$ but $v_1( \mc{D}' ) = u'$, so $\norm{ v_1(\mc{D})  - v_1(\mc{D}') } = \sqrt{2}$.  Thus the global sensitivity does not scale with the number of data points, so as $n$ increases the variance of the noise required by the sensitivity method will not decrease.   An alternative to global sensitivity is smooth sensitivity~\citep{NRS07}.  Except for special cases, such as the sample median, smooth sensitivity is difficult to compute for general functions. A third method for computing private, approximate solutions to high-dimensional optimization problems is objective perturbation~\citep{ChaudhuriMS:11erm}; to apply this method, we require the optimization problems to have certain properties (namely, strong convexity and bounded norms of gradients), which do not apply to PCA.

\subsection{Main results}

Our theoretical results are sample complexity bounds for \ouralgorithm\ and \modsulq\ as well as a general lower bound on the sample complexity for any $\priveps$-differentially private algorithm.  These results show that the \ouralgorithm\ is nearly optimal in terms of the scaling of the sample complexity with respect to the data dimension $d$, privacy parameter $\priveps$, and eigengap $\Delta$.  We further show that \modsulq\ requires more samples as a function of $d$, despite having a slightly weaker privacy guarantee.  Proofs are presented in Sections~\ref{sec:ppcaanalysis} and~\ref{sec:sulqanalysis}.

Even though both algorithms can output the top-$k$ PCA subspace for general $k \le d$, we prove results for the case $k = 1$.  %
 Finding the scaling behavior of the sample complexity with $k$ is an interesting open problem that we leave for future work; challenges here are finding the right notion of approximation of the PCA, and extending the theory using packings of Grassman or Stiefel manifolds.

\begin{theorem}
\label{thm:sulq:priv}
For the $\beta$ in \eqref{eq:sulqsd}
Algorithm \modsulq\ is $(\priveps,\delta)$ differentially private.
\end{theorem}

\begin{theorem}
\label{thm:pppca:privacy}
Algorithm \ouralgorithm\ is $\priveps$-differentially private.
\end{theorem}

The fact that these two algorithms are differentially private follows from some simple calculations.
Our first sample complexity result provides an upper bound on the number of samples required by \ouralgorithm\ to guarantee a certain level of privacy and accuracy.  The sample complexity of \ouralgorithm\ grows linearly with the dimension $d$, inversely with $\priveps$, and inversely with the correlation gap $(1 - \rho)$ and eigenvalue gap $\Delta$. These sample complexity results hold for $k=1$.

\begin{theorem}[Sample complexity of \ouralgorithm]
If
	\begin{align*}	
	n > \frac{d}{\priveps \Delta ( 1 - \rho )  } 
		\left( 4 \frac{\log(1/\eta)}{d} 
			+ 2 \log \frac{8 \lambda_1}{(1 - \rho^2) \Delta} 
			\right),
	\end{align*}
then the top PCA direction $v_1$ and the output of \ouralgorithm\ $\hat{v}_1$ with privacy parameter $\priveps$ satisfy
	\begin{align*} 
	\Pr( \left|\ip{v_1}{\hat{v}_1} \right| > \rho) \geq 1 - \eta.
	\end{align*}
That is, 
\ouralgorithm\ is a $(\rho,\eta)$-close approximation to PCA.
\label{thm:upperexp}
\end{theorem}

Our second result shows a lower bound on the number of samples required by {\em{any}} $\priveps$-differentially-private algorithm to guarantee a certain level of accuracy for a large class of datasets, and uses proof techniques in~\cite{CH11, CH12}. 

\begin{theorem}[Sample complexity lower bound]
Fix $d \ge 3$, $\priveps$, $\Delta \leq \frac{1}{2}$ and let
	\begin{align*}
	1 - \phi = \exp\left( -2 \cdot \frac{\ln 8+ \ln(1 + \exp(d)) }{ d - 2 } \right).
	\end{align*}
For any $\rho \geq 1 - \frac{1 - \phi}{16}$, no $\priveps$-differentially private algorithm $\mc{A}$ can approximate PCA with expected utility greater than $\rho$ on all databases with $n$ points in dimension $d$ having eigenvalue gap $\Delta$, where
	\begin{align*}
	n < \frac{ d }{ \priveps \Delta } \max \left\{ 
		1, 
		\sqrt{\frac{ 1 - \phi }{80 (1 - \rho) }}
		\right\}.
	\end{align*}
\label{thm:lowerexp}
\end{theorem}

Theorem \ref{thm:upperexp} shows that if $n$ scales like $\frac{d}{ \priveps \Delta (1 - \rho) } \log \frac{1}{1 - \rho^2}$ then \ouralgorithm\ produces an approximation $\hat{v}_1$ that has correlation $\rho$ with $v_1$, whereas Theorem \ref{thm:lowerexp} shows that $n$ must scale like $\frac{d}{\priveps \Delta \sqrt{(1 - \rho)}}$ for any $\priveps$-differentially private algorithm.  In terms of scaling with $d$, $\priveps$ and $\Delta$, the upper and lower bounds match, and they also match up to square-root factors with respect to the correlation.  By contrast, the following lower bound on the number of samples required by \modsulq\ to ensure a certain level of accuracy shows that \modsulq\ has a less favorable scaling with dimension.

\begin{theorem}[Sample complexity lower bound for \modsulq]
There are constants $c$ and $c'$ such that if
	\begin{align*}
	n < c \frac{d^{3/2} \sqrt{\log(d/\delta)}}{ \priveps } (1 - c' (1 - \rho)),
	\end{align*}
then there is a dataset of size $n$ in dimension $d$ such that the top PCA direction $v$ and the output $\hat{v}$ of \modsulq\ satisfy $\expe{ \left| \ip{\hat{v}_1}{v_1} \right| } \le \rho$.
\label{thm:lowersulq}
\end{theorem}
Notice that the dependence on $n$ grows as $d^{3/2}$ in SULQ as opposed to $d$ in \ouralgorithm. Dimensionality reduction via PCA is often used in applications where the data points occupy a low dimensional space but are presented in high dimensions.  These bounds suggest that \ouralgorithm\ is better suited to such applications than \modsulq.

\section{Analysis of \ouralgorithm}
\label{sec:ppcaanalysis}

In this section we provide theoretical guarantees on the performance of \ouralgorithm.  The proof of Theorem \ref{thm:pppca:privacy} follows from the results on the exponential mechanism \citep{MT07}.  To find the sample complexity of \ouralgorithm\ we bound the density of the Bingham distribution, leading to a sample complexity for $k = 1$ that depends on the gap $\lambda_1 - \lambda_2$ between the top two eigenvalues.  We also prove a general lower bound on the sample complexity that holds for any $\priveps$-differentially private algorithm.  The lower bound matches our upper bound up to log factors, showing that \ouralgorithm\ is nearly optimal in terms of the scaling with dimension, privacy $\priveps$, and utility $\qangle(\cdot)$.

\subsection{Privacy guarantee}

We first give a proof of Theorem \ref{thm:pppca:privacy}.

\begin{proof}
Let $X$ be a data matrix whose $i$-th column is $x_i$ and $A = \frac{1}{n} X X^T$.  The \textsf{PP-PCA} algorithm is the exponential mechanism of~\cite{MT07} applied to the score function $n \cdot v^T A v$.  Consider $X' = [x_1\ x_2\ \cdots\ x_{n-1}\ x_{n}']$ differing from $X$ in a single column and let $A' = \frac{1}{n} X' X'^T$.  We have
	\begin{align*}
	\max_{v \in \Sp^{d-1}} \left| n \cdot v^T A' v - n \cdot v^T A v \right| 
		&\le  \left| v^T (x_{n}'x_{n}'^T - x_{n} x_{n}^T) v \right| \\
		&\le  \left| \norm{v^T x_{n}'}^2 - \norm{v^T x_{n}}^2 \right| \\
		&\le 1.
	\end{align*}
The last step follows because $\| x_i \| \le 1$ for all $i$. The result now follows immediately from \citet[Theorem 6]{MT07}.
\end{proof}

\subsection{ Upper bound on utility }

The results on the exponential mechanism bound the gap between the value of the function $\qnorm(\hat{v}_1) = n \cdot \hat{v}_1^T A \hat{v}_1$ evaluated at the output $\hat{v}_1$ of the mechanism and the optimal value $q(v_1) = n \cdot \lambda_1$.  We derive a bound on the correlation $\qangle(\hat{v}_1) = |\ip{\hat{v}_1}{v_1}|$ via geometric arguments.

\begin{lemma}[Lemmas 2.2 and 2.3 of~\cite{Ball:97convex}]
\label{lem:cap_surface_area}
Let $\mu$ be the uniform measure on the unit sphere $\Sp^{d-1}$.  For any $x \in \Sp^{d-1}$ and $0 \le c < 1$ the following bounds hold:
	\begin{align*}
	\frac{1}{2} \exp\left(- \frac{d-1}{2} \log \frac{2}{1 - c} \right)
	\le
	\mu\left( \left\{ v \in \Sp^{d-1} : \ip{v}{x} \ge c \right\} \right) 
	\le
	\exp\left( -d c^2/2 \right).	  
	\end{align*}
\end{lemma}

We are now ready to provide a proof of Theorem \ref{thm:upperexp}.

\begin{proof}
Fix a privacy level $\priveps$, target correlation $\rho$, and probability $\eta$.  Let $X$ be the data matrix and $B = (\priveps/2) XX^T$ and
	\begin{align*}
	\mc{U}_{\rho} = \left\{ u : |\ip{u}{v_1}| \ge \rho \right\}.
	\end{align*}
be the union of the two spherical caps centered at $\pm v_1$.  Let $\overline{\mc{U}}_{\rho}$ denote the complement of $\mc{U}_{\rho}$ in $\Sp^{d-1}$.

An output vector $\hat{v}_1$ is ``good'' if it is in $\mc{U}_{\rho}$.  We first give some bounds on the score function $\qnorm(u)$ on the boundary between $\mc{U}_{\rho}$ and $\overline{\mc{U}}_{\rho}$, where $\ip{u}{v_1} = \pm {\rho}$.  On this boundary, the function $\qnorm(u)$ is maximized when $u$ is a linear combination of $v_1$ and $v_2$, the top two eigenvectors of $A$.  It minimized when $u$ is a linear combination of $v_1$ and $v_d$.  Therefore
	\begin{align}
	\qnorm(u) &\le \frac{n \priveps}{2} ( \rho^2 \lambda_1 + (1 - \rho^2) \lambda_2) \qquad u \in \overline{\mc{U}}_{\rho} 
		\label{eq:outcapub} \\
	\qnorm(u) &\ge \frac{n \priveps}{2} ( \rho^2 \lambda_1 + (1 -\rho^2) \lambda_d) \qquad u \in \mc{U}_{\rho}.
		\label{eq:incaplb}
	\end{align}

Let $\mu(\cdot)$ denote the uniform measure on the unit sphere.  Then fixing an $0 \le b < 1$, using \eqref{eq:outcapub}, \eqref{eq:incaplb}, and the fact that $\lambda_d \ge 0$,
	\begin{align}
	\prob{ \overline{\mc{U}}_{\rho} }
	&\le \frac{ \prob{ \overline{\mc{U}}_{\rho} } }{ \prob{ \mc{U}_{\sigma} } }  
		\nonumber \\
	&= \frac{ \frac{1}{ \tensor*[_1]{F}{_1}\left(\frac{1}{2} k, \frac{1}{2} m, B\right) }
			\int_{\overline{\mc{U}}_{\rho}} \exp\left( u^T B u \right) d\mu
		}{
		\frac{1}{ \tensor*[_1]{F}{_1}\left(\frac{1}{2} k, \frac{1}{2} m, B\right) }
			\int_{\mc{U}_{\sigma}} \exp\left( u^T B u \right) d\mu
		} \nonumber \\
	&\le \frac{
		\exp\left( n (\priveps/2) \left( 
			\rho^2 \lambda_1 + (1 - \rho^2) \lambda_2  \right) 
			\right) \cdot \mu\left( \overline{\mc{U}}_{\rho} \right)
		}{
		\exp\left( n (\priveps/2) \left( 
			\sigma^2 \lambda_1 + (1 - \sigma^2) \lambda_d  \right) 
			\right) \cdot \mu\left( \mc{U}_{\sigma} \right)
		} \nonumber \\
	&\le 
		\exp\left( - \frac{n\priveps}{2} \left( 
			\sigma^2 \lambda_1 - (\rho^2 \lambda_1 +  (1 - \rho^2) \lambda_2)  
			\right) 
			\right) 
			\cdot 
			\frac{ \mu\left( \overline{\mc{U}}_{\rho} \right)
				}{
				\mu\left( \mc{U}_{\sigma} \right)
				}.
		\label{eq:exp:ub2}
	\end{align}
Applying the lower bound from Lemma \ref{lem:cap_surface_area} to the denominator of \eqref{eq:exp:ub2} and the upper bound $\mu\left( \overline{\mc{U}}_{\rho} \right) \le 1$ yields
\begin{align}
	\prob{ \overline{\mc{U}}_{\rho} }  & \le	\exp\left( - \frac{n\priveps}{2} \left( 
			\sigma^2 \lambda_1 - (\rho^2 \lambda_1 +  (1 - \rho^2) \lambda_2)  
			\right) 
			\right) 
			\cdot \exp\left( \frac{d-1}{2} \log \frac{2}{1 - \sigma} \right).
\label{eq:exp:exponent}
\end{align}

We must choose a $\sigma^2 > \rho^2$ to make the upper bound smaller than $1$.  More precisely,
	\begin{align*}
	\sigma^2 &> \rho^2 +  (1 - \rho^2) \frac{\lambda_2}{\lambda_1} \\
	1 - \sigma^2 &< (1 - \rho^2) \left( 1 - \frac{\lambda_2}{\lambda_1} \right).
	\end{align*}
For simplicity, choose 
	\begin{align*}
	1 - \sigma^2 = \frac{1}{2} (1 - \rho^2) \left( 1 - \frac{\lambda_2}{\lambda_1} \right).
	\end{align*}
So that 
	\begin{align*}
	\sigma^2 \lambda_1 - (\rho^2 \lambda_1 +  (1 - \rho^2) \lambda_2)  
	&= (1 - \rho^2) \lambda_1 - (1 - \sigma^2) \lambda_1 - (1 - \rho^2) \lambda_2 \\
	&= (1 - \rho^2) \left( \lambda_1
		- \frac{1}{2} (\lambda_1 - \lambda_2)
		- \lambda_2
		\right) \\
	&= \frac{1}{2} (1 - \rho^2) (\lambda_1 - \lambda_2)
	\end{align*}
and
	\begin{align*}
	\log \frac{2}{1 - \sigma}
	&< \log \frac{4}{1 - \sigma^2} \\
	&= \log \frac{8 \lambda_1}{(1 - \rho^2)(\lambda_1 - \lambda_2)}.
	\end{align*}
Setting the right hand side of \eqref{eq:exp:exponent} less than $\eta$ yields
	\begin{align*}
	\frac{n \priveps}{4} (1 - \rho^2) (\lambda_1 - \lambda_2) > \log\frac{1}{\eta} + \frac{d-1}{2} \log \frac{8 \lambda_1}{(1 - \rho^2)(\lambda_1 - \lambda_2)}.
	\end{align*}
Because $1 - \rho < 1 - \rho^2$, if we choose
	\begin{align*}
	n > \frac{d}{\priveps ( 1 - \rho )(\lambda_1 - \lambda_2) } 
		\left( 4 \frac{\log(1/\eta)}{d} 
			+ 2 \log \frac{8 \lambda_1}{(1 - \rho^2)(\lambda_1 - \lambda_2)} 
			\right),
	\end{align*}
then the output of \ouralgorithm\ will produce a $\hat{v}_1$ such that 
	\begin{align*}
	\prob{ \left| \ip{\hat{v}_1}{v_1} \right| < \rho }
		< \eta.
	\end{align*}
\end{proof}

\subsection{ Lower bound on utility }

We now turn to a general lower bound on the sample complexity for any differentially private approximation to PCA.  We construct $K$ databases which differ in a small number of points whose top eigenvectors are not too far from each other.  For such a collection, Lemma \ref{lem:lb:expcorr} shows that for any differentially private mechanism, the average correlation over the collection cannot be too large.  That is, any $\priveps$-differentially private mechanism cannot have high utility on all $K$ data sets.  The remainder of the argument is to construct these $K$ data sets.

The proof uses some simple eigenvalue and eigenvector computations.  A matrix of positive entries
	\begin{align}
	A = \left( 
		\begin{array}{cc}
		a & b \\
		b & c
		\end{array}
		\label{eq:twovec:mat}
		\right)
	\end{align}
has characteristic polynomial
	\begin{align*}
	\det(A - \lambda I) = \lambda^2 - (a + c) \lambda + (ac - b^2)
	\end{align*}
and eigenvalues
	\begin{align*}
	\lambda &= \frac{1}{2} (a + c) \pm \frac{1}{2} \sqrt{ (a+c)^2 - 4 (ac - b^2) } 
		 \\
	&= \frac{1}{2} (a + c) \pm \frac{1}{2} \sqrt{ (a - c)^2 + 4b^2 }.
	\end{align*}
The eigenvectors are in the directions $(b, -(a - \lambda))^T$.

We will also need the following Lemma, which is proved in the Appendix.

\begin{lemma}[Simple packing set]
\label{lem:simplepack}
For $\phi \in [ (2 \pi d)^{-1/2}, 1)$, there exists a set of 
	\begin{align}
	K = \frac{1}{8} \exp\left( (d-1) \log \frac{1}{\sqrt{1 - \phi^2}} \right)
	\label{eq:simplepack}
	\end{align}
vectors $\mc{C}$ in $\Sp^{d-1}$ such that for any pair $\mu, \nu \in \mc{C}$, the 
inner product between them is upper bounded by $\phi$:
	\begin{align*}
	\left| \ip{ \mu }{ \nu } \right| \le \phi.
	\end{align*}
\end{lemma}

The following Lemma gives a lower bound on the expected utility averaged over a set of databases which differ in a ``small'' number of elements.

\begin{lemma}
Let $\mc{D}_1, \mc{D}_2, \ldots, \mc{D}_K$ be $K$ databases which differ in the value of at most $\frac{\ln(K-1)}{\priveps}$ points, and let $u_1, \ldots, u_K$ be the top eigenvectors of $\mc{D}_1, \mc{D}_2, \ldots, \mc{D}_K$.  If $\calA$ is any $\priveps$-differentially private algorithm, then, 
	\begin{align*} 
	\sum_{i=1}^{K} 
		\E_{\calA}\left[ \left|\ip{\calA(\mc{D}_i)}{u_i} \right| \right]
	\le 
		K \left( 1 - \frac{1}{16} (1 - \max \left| \ip{ u_i }{ u_j } \right|) \right).
	\end{align*}
\label{lem:lb:expcorr}
\end{lemma}

\begin{proof}
Let
	\begin{align*}
	t = \min_{i \ne j} (\norm{ u_i - u_j }, \norm{u_i + u_j} ),
	\end{align*}
and $\mc{G}_i$ be the ``double cap'' around $\pm u_i$ of radius $t/2$:
	\begin{align*}
	\mc{G}_i = \left\{ u : \norm{u - u_i} < t/2 \right\}
		\cup \left\{ u : \norm{u + u_i} < t/2 \right\}.
	\end{align*}	
We claim that
	\begin{align}
	\label{eqn:probpacking}
	\sum_{i=1}^{K} \pr_{\calA}(\calA(\mc{D}_i) \notin G_i) \ge \frac{1}{2} (K - 1) .
	\end{align} 
The proof is by contradiction.  Suppose the claim is false.  Because all of the caps $\mc{G}_i$ are disjoint, and applying the definition of differential privacy,
	\begin{align*}
	\frac{1}{2}(K-1) 
	&> \sum_{i=1}^{K} \pr_{\calA}(\calA(\mc{D}_i) \notin G_i) \\
	&\ge \sum_{i=1}^{K} \sum_{i' \neq i} \pr_{\calA}(\calA(\mc{D}_i) \in G_{i'}) \\
	&\ge \sum_{i=1}^{K} \sum_{i' \neq i} 
		e^{-\priveps \cdot \ln(K-1)/\priveps} 
			\pr_{\calA}(\calA(\mc{D}_{i'}) \in G_{i'})  \\
	&\ge (K-1) \cdot \frac{1}{K-1} \cdot \sum_{i=1}^{K} 
		\pr_{\calA}(\calA(\mc{D}_i) \in G_i) \\
    &\ge K - \frac{1}{2}(K-1),
	\end{align*}
which is a contradiction, so \eqref{eqn:probpacking} holds. Therefore by the Markov inequality
	\begin{align*}
	\sum_{i=1}^{K} 
		\E_{\calA}\left[ \min(\norm{ \calA(\mc{D}_i) - u_i }^2, \norm{\calA(\mc{D}_i) + u_i}^2) \right] 
	&\ge \sum_{i=1}^{K} \pr(\calA(\mc{D}_i) \notin G_i) \cdot \frac{t^2}{4} \\
	&\ge \frac{1}{8} (K-1) t^2.
	\end{align*}
Rewriting the norms in terms of inner products shows
	\begin{align*}
	2K - 2 \sum_{i=1}^{K} 
		\E_{\calA}\left[ \left| \ip{ \calA(\mc{D}_i)}{ u_i } \right| \right] 
	&\ge \frac{1}{8} (K-1) \left(2 - 2 \max \left| \ip{ u_i }{ u_j } \right| \right),
	\end{align*} 
so
	\begin{align*}
	\sum_{i=1}^{K} \E_{\calA}\left[ \left| \ip{ \calA(\mc{D}_i)}{ u_i } \right| \right]
	&\le K \left( 1 - \frac{1}{8} \frac{K-1}{K} (1 - \max \left| \ip{ u_i }{ u_j } \right|) \right) \\
	&\le K \left( 1 - \frac{1}{16} (1 - \max \left| \ip{ u_i }{ u_j } \right|)\right).
	\end{align*}
\end{proof}

We can now prove Theorem \ref{thm:lowerexp}.

\begin{proof}
From Lemma \ref{lem:lb:expcorr}, given a set of $K$ databases differing in $\frac{ \ln(K-1) }{ \priveps }$ points with top eigenvectors $\{u_i : i =1, 2, \ldots, K\}$, for at least one database $i$,
	\begin{align*}
	\E_{\mc{A}}\left[ \left| \ip{ \mc{A}(\mc{D}_i) }{ u_i } \right| \right]
	&\le 1 - \frac{1}{16} \left( 1 - \max \left| \ip{ u_i }{ u_j } \right| \right)
	\end{align*}
for any $\priveps$-differentially private algorithm.  Setting the left side equal to some target $\rho$,
	\begin{align}
	1 - \rho \ge \frac{1}{16} \left( 1 - \max \left| \ip{ u_i }{ u_j } \right| \right).
	\label{eq:lb:corrcond}
	\end{align}
So our goal is construct these data bases such that the inner product between their eigenvectors is small.

Let $y = e_d$, the $d$-th coordinate vector, and let $\phi \in ( (2 \pi d)^{-1/2},1)$.  Lemma \ref{lem:simplepack} shows that there exists a packing $\mc{W} = \{w_1, w_2, \ldots, w_K\}$ of the sphere $\Sp^{d-2}$ spanned by $\{e_1,e_2,\ldots,e_{d-1}\}$ such that $\max_{i \ne j} |\ip{w_i}{w_j}| \le \phi$, where
	\begin{align*}
	K = \frac{1}{8} (1 - \phi)^{-(d-2)/2}.
	\end{align*}  
Choose $\phi$ such that $\ln(K-1) = d$.  This means 
	\begin{align*}
	1 - \phi &= \exp\left( -2 \cdot \frac{\ln 8+ \ln(1 + \exp(d)) }{ d - 2 } \right).
	\end{align*}
The right side is minimized for $d = 3$ but this leads to a weak lower bound $1 - \phi > 3.5 \times 10^{-5}$.  By contrast, for $d = 100$, the bound is $1 - \phi > 0.12$. In all cases, $1 - \phi$ is at least a constant value.
	
We construct a database with $n$ points for each $w_i$.  Let $\beta = \frac{d}{ n \priveps }$. For now, we assume that $\beta \leq \Delta \leq \frac{1}{2}$. The other case, when $\beta \geq \Delta$ will be considered later.  Because $\beta \leq \Delta$, we have
	\begin{align*}
	n > \frac{d}{\priveps \Delta}.
	\end{align*}
The construction uses a parameter $0 \leq m \leq 1$ that will be set as a function of the eigenvalue gap $\Delta$.  We will derive conditions on $n$ based on the requirements on $d$, $\priveps$, $\rho$, and $\Delta$.  For $i = 1, 2, \ldots, K$ let the data set $\mc{D}_i$ contain 
	\begin{itemize}
	\item $n ( 1 - \beta )$ copies of $\sqrt{m} y$
	\item $n \beta$ copies of $z_i = \frac{1}{\sqrt 2} y + \frac{1}{\sqrt 2} w_i$.
	\end{itemize}
Thus datasets $\mc{D}_i$ and $\mc{D}_j$ differ in the values of $n \beta = \frac{\ln(K -1)}{n \priveps}$ individuals. The second moment matrix $A_i$ of $\mc{D}_i$ is
	\begin{align*}
	A_i = ((1 - \beta)m + \frac{1}{2} \beta) y y^T + \frac{1}{2} \beta ( w_i^T y + y w_i^T) + \frac{1}{2} \beta w_i w_i^T.
	\end{align*}	
By choosing an basis containing $y$ and $w_i$, we can write this as
	\begin{align*}
	A_i = \left[ 
		\begin{array}{ccc}
		(1 - \beta)m + \frac{1}{2} \beta &  \frac{1}{2} \beta & \mbf{0} \\
		 \frac{1}{2} \beta & \frac{1}{2} \beta & \mbf{0} \\
		\mbf{0} & \mbf{0} & \mbf{0}
		\end{array}
		\right].
	\end{align*}
This is in the form \eqref{eq:twovec:mat}, with $a = (1 - \beta)m + \frac{1}{2} \beta$, $b = \frac{1}{2} \beta$, and $c = \frac{1}{2} \beta$.   

The matrix $A_i$ has two nonzero eigenvalues given by
	\begin{align}
	\lambda &= \frac{1}{2} (a + c) + \frac{1}{2} \sqrt{(a - c)^2 + 4b^2 }, \label{eq:lb:topeig} \\
	\lambda' &= \frac{1}{2} (a + c) - \frac{1}{2}  \sqrt{(a - c)^2 + 4b^2 } \nonumber ,
	\end{align}
The gap $\Delta$ between the top two eigenvalues is:
	\begin{align*} 
	\Delta = \sqrt{(a - c)^2 + 4b^2} = \sqrt{ m^2(1 - \beta)^2 + \beta^2}.
	\end{align*}
We can thus set $m$ in the construction to ensure an eigengap of $\Delta$:
	\begin{align}
	m = \frac{  \sqrt{ (\Delta^2 - \beta^2) } }{ 1 - \beta } .
	\label{eq:lb:mvalue}
	\end{align}

The top eigenvector of $A_i$ is given by 
	\begin{align*}
	u_i = \frac{b}{\sqrt{b^2 + (a - \lambda)^2}} y + \frac{(a - \lambda)}{ \sqrt{b^2 + (a - \lambda)^2}} w_i.
	\end{align*}
where $\lambda$ is given by \eqref{eq:lb:topeig}.  Therefore
	\begin{align}
	\max_{i \ne j} \left| \ip{u_i}{u_j} \right| 
	&\le  \frac{b^2}{b^2 + (a - \lambda)^2} + \frac{(a - \lambda)^2}{ b^2 + (a - \lambda)^2} \max_{i \ne j} \left| \ip{w_i}{w_j} \right| \nonumber \\
	&\le 1 - \frac{(a - \lambda)^2}{ b^2 + (a - \lambda)^2} (1 - \phi).
	\label{eq:lb:absdotprod}
	\end{align}
To obtain an upper bound on $\max_{i \ne j} \left| \ip{u_i}{u_j} \right|$ we must lower bound $\frac{(a - \lambda)^2}{b^2 + (a - \lambda)^2}$.  

Since $x/(\nu + x)$ is monotonically increasing in $x$ when $\nu > 0$, we will find a lower bound on $(a - \lambda)$. Observe that from \eqref{eq:lb:topeig},
	\begin{align*} 
	\lambda - a = \frac{b^2}{\lambda - c}.
	\end{align*}
So to lower bound $\lambda - a$ we need to upper bound $\lambda - c$.  We have
	\begin{align*}
	\lambda - c = \frac{1}{2} (a - c) + \frac{1}{2} \Delta = \frac{1}{2} \left( (1 - \beta) m + \Delta \right).
	\end{align*}
Because $b = \beta/2$,
	\begin{align*}
	(\lambda - a)^2 &> \left( \frac{ \beta^2 }{ 2 ( (1 - \beta) m  + \Delta) } \right)^2 = \frac{\beta^4}{4( (1 - \beta) m + \Delta)^2}.
	\end{align*}
Now, 
	\begin{align}
	\frac{(a - \lambda)^2}{b^2 + (a - \lambda)^2}
	&> \frac{ \beta^4 }{  \beta^2( (1 - \beta) m + \Delta)^2 + \beta^4 } \nonumber \\
	&= \frac{ \beta^2 }{ \beta^2 +  ((1 - \beta) m + \Delta)^2 }  \nonumber \\
	&> \frac{ \beta^2 }{ 5 \Delta^2 }, \label{eq:lb:coefflb}
	\end{align}
where the last step follows by plugging in $m$ from \eqref{eq:lb:mvalue} and using the fact that $\beta \leq \Delta$.
Putting it all together, we have from \eqref{eq:lb:corrcond}, \eqref{eq:lb:absdotprod}, and \eqref{eq:lb:coefflb}, and using the fact that $\phi$ is such that $\ln(K-1) = d$ and $\beta = \frac{ d }{ n \priveps }$, 
	\begin{align*}
	1 - \rho &\ge \frac{1}{16} \cdot \frac{(a - \lambda)^2}{ b^2 + (a - \lambda)^2} (1 - \phi) \\
	&> \frac{ 1 - \phi }{ 80 } \frac{ \beta^2  }{ \Delta^2 }  \\
	&= \frac{ 1 - \phi }{ 80 } \cdot \frac{ d^2 }{ n^2  \priveps^2 \Delta^2  },
	\end{align*}
which implies
	\begin{align*}
	n > \frac{ d }{ \priveps \Delta } \sqrt{ \frac{ 1 - \phi }{ 80 (1 - \rho) } }.
	\end{align*}
Thus for $\beta \leq \Delta \le 1/2$, any $\priveps$-differentially private algorithm needs $\Omega\left( \frac{ d }{ \priveps \Delta \sqrt{ 1 - \rho } } \right)$ points to get expected inner product $\rho$ on all data sets with eigengap $\Delta$.

We now consider the case where $\beta > \Delta$. We choose a slightly different construction here. The $i$-th database now consists of $n(1 - \beta)$ copies of the $0$ vector, and $n \beta$ copies of $\frac{\Delta}{\beta} w_i$. Thus, every pair of databases differ in the values of $n \beta = \frac{\ln(K-1)}{\priveps}$ people, and the eigenvalue gap between the top two eigenvectors is $\beta \cdot \frac{\Delta}{\beta} = \Delta$.

As the top eigenvector of the $i$-th database is $u_i = w_i$,  
	\begin{align*} 
	\max_{i \neq j} |\ip{u_i}{u_j}| = \max_{i \neq j} |\ip{w_i}{w_j}| \leq \phi. 
	\end{align*}
Combining this with \eqref{eq:lb:corrcond}, we obtain
	\begin{align*} 
	1 - \rho \geq \frac{1}{16}(1 - \phi),
	\end{align*}
which provides the additional condition in the Theorem.
\end{proof}

\section{Analysis of \modsulq}

\label{sec:sulqanalysis}
In this section we provide theoretical guarantees on the performance of the \modsulq\ algorithm.  Theorem \ref{thm:sulq:priv} shows that \modsulq\ is $(\priveps,\delta)$-differentially private.  Theorem \ref{thm:sulq:lb} provides a lower bound on the distance between the vector released by \modsulq\ and the true top eigenvector in terms of the privacy parameters $\priveps$ and $\delta$ and the number of points $n$ in the data set.  This implicitly gives a lower bound on the sample complexity of \modsulq.  We provide some graphical illustration of this tradeoff.

The following upper bound will be useful for future calculations : for two unit vectors $x$ and $y$, 
	\begin{align}
	\sum_{1\le i \le j \le d} (x_i x_j - y_i y_j)^2 \le 2.
	\label{eq:worstdiff}
	\end{align}
Note that this upper bound is achievable by setting $x$ and $y$ to be orthogonal elementary vectors.

\subsection{ Privacy guarantee }

We first justify the choice of $\beta^2$ in the \textsf{MOD-SULQ} algorithm by proving Theorem \ref{thm:sulq:priv}.

\begin{proof}
Let $B$ and $\hat{B}$ be two independent symmetric random matrices where $\{B_{ij} : 1 \le i \le j \le d \}$ and $\{\hat{B}_{ij} : 1 \le i \le j \le d\}$ are each sets of i.i.d. Gaussian random variables with mean $0$ and variance $\beta^2$.  Consider two data sets $\mc{D} = \{x_i : i = 1, 2, \ldots, n\}$ and $\hat{\mc{D}} = \mc{D}_1 \cup \{\hat{x}_n \} \setminus \{x_n\}$ and let $A$ and $\hat{A}$ denote their second moment matrices.  Let $G = A + B$ and $\hat{G} = \hat{A} + \hat{B}$.  We first calculate the log ratio of the densities of $G$ and $\hat{G}$ at a point $H$:
	\begin{align*}
	\log \frac{ f_{G}(H) }{ f_{\hat{G}}(H) }
	&= \sum_{1 \le i \le j \le d} \left( - \frac{1}{2 \beta^2} (H_{ij} - A_{ij})^2 
		+ \frac{1}{2\beta^2} (H_{ij} - \hat{A}_{ij})^2 \right) \\
	&= \frac{1}{2 \beta^2} \sum_{1 \le i \le j \le d} \left( 
		\frac{2}{n} (H_{ij} - A_{ij}) ({x}_{n,i} {x}_{n,j} - \hat{x}_{n,i} \hat{x}_{n,j})
		+ \frac{1}{n^2} (\hat{x}_{n,i} \hat{x}_{n,j} - x_{n,i} x_{n,j})^2
		\right).
	\end{align*}
	
From \eqref{eq:worstdiff} the last term is upper bounded by $2/n^2$. To upper bound the first term,
	\begin{align*}
	\sum_{1 \le i \le j \le d} |\hat{x}_{n,i} \hat{x}_{n,j} - x_{n,i} x_{n,j}|
	&\le 2 \max_{a : \norm{a} \le 1} \sum_{1 \le i \le j \le d} a_i a_j \\ 
	&\le 2 \cdot \frac{1}{2} (d^2 + d) \cdot \frac{1}{d} \\
	&= d + 1.
	\end{align*}
Note that this bound is not too loose---by taking $\hat{x} = d^{-1/2} \mbf{1}$ and $x = (1, 0, \ldots,0)^T$, this term is still linear in $d$.
	
Then for any measurable set $\mc{S}$ of matrices,
	\begin{align}
	\prob{ G \in \mc{S} } \le \exp\left( \frac{1}{2 \beta^2} 
			\left( \frac{2}{n} (d + 1) \gamma + \frac{3}{n^2} \right) 
			\right)
		\prob{ \hat{G} \in \mc{S} } + \prob{ B_{ij} > \gamma\  {\text{for all\;}} i,j }.
	\label{eq:sulq:dp}
	\end{align}

To handle the last term, use a union bound over the $(d^2 + d)/2$ variables $\{B_{ij}\}$ together with the tail bound, which holds for $\gamma > \beta$:
	\begin{align*}
	\prob{ B_{ij} > \gamma } \le \frac{1}{ \sqrt{2 \pi} } e^{- \gamma^2/2 \beta^2 }.
	\end{align*}
Thus setting $\prob{ B_{ij} > \gamma\  {\text{for some\;}} i,j } = \delta$ yields the condition
	\begin{align*}
	\delta = \frac{d^2 + d}{2 \sqrt{2 \pi} } e^{- \gamma^2/2 \beta^2 }.
	\end{align*}
Rearranging to solve for $\gamma$ gives
	\begin{align*}
	\gamma = \max\left( \beta, \beta \sqrt{2 \log \left( \frac{d^2 + d}{ \delta 2 \sqrt{2 \pi} } \right) } \right) = \beta \sqrt{2 \log \left( \frac{d^2 + d}{ \delta 2 \sqrt{2 \pi} } \right) }
	\end{align*}
for $d > 1$ and $\delta < 3/\sqrt{2 \pi e}$. This then gives an expression for $\priveps$ to make \eqref{eq:sulq:dp} imply $(\priveps,\delta)$ differential privacy:
	\begin{align*}
	\priveps &= \frac{1}{2 \beta^2} 
			\left( \frac{2}{n} (d + 1) \gamma + \frac{2}{n^2}   \right)\\
		&= \frac{1}{2 \beta^2} 
			\left( \frac{2}{n} (d + 1) \beta \sqrt{2 \log \left( \frac{d^2 + d}{ \delta 2 \sqrt{2 \pi} } \right) } + \frac{2}{n^2} \right).
	\end{align*}
Solving for $\beta$ using the quadratic formula yields the particularly messy expression in \eqref{eq:sulqsd}:
\begin{align*}
\beta &=  \frac{d + 1}{2 n \priveps} 
				\sqrt{ 2 \log \left( \frac{d^2 + d}{ \delta 2 \sqrt{2 \pi} } \right) }
+ \frac{1}{2 n \priveps} \left( 2 (d+1)^2 \log \left( \frac{d^2 + d}{ \delta 2 \sqrt{2 \pi}}\right) + 4 \priveps \right)^{1/2}  \\
& \le  \frac{d+1}{ n \priveps} \sqrt{ 2 \log \left( \frac{d^2 + d}{ \delta 2 \sqrt{2 \pi} } \right) }
+ \frac{1}{ \sqrt{\priveps} n}.
\end{align*}
\end{proof}

\subsection{Proof of Theorem \ref{thm:lowersulq}}

In this section we provide theoretical guarantees on the performance of the \modsulq\ algorithm.  Theorem \ref{thm:sulq:priv} shows that \modsulq\ is $(\priveps,\delta)$-differentially private.  Theorem \ref{thm:sulq:lb} provides a lower bound on the distance between the vector released by \modsulq\ and the true top eigenvector in terms of the privacy parameters $\priveps$ and $\delta$ and the number of points $n$ in the data set.  This implicitly gives a lower bound on the sample complexity of \modsulq.  We provide some graphical illustration of this tradeoff.  The main tool in our lower bound is a generalization by~\cite{Yu:97assouad} of an information-theoretic inequality due to Fano.

\begin{theorem}[Fano's inequality~\citep{Yu:97assouad}] \label{thm:fano}
Let $\mc{R}$ be a set and $\Theta$ be a parameter space with a pseudo-metric $d(\cdot)$.  Let $\mc{F}$ be a set of $r$ densities $\{f_1,\ldots,f_r\}$ on $\mc{R}$ corresponding to parameter values $\{\theta_1,\ldots,\theta_r\}$ in $\Theta$.  Let $X$ have distribution $f \in \mc{F}$ with corresponding parameter $\theta$ and let $\hat{\theta}(X)$ be an estimate of $\theta$. If, for all $i$ and $j$
	\begin{align*}
	d(\theta_i, \theta_j) \geq \tau
	\end{align*}
and 
	\begin{align*}
	\kldiv{ f_i }{ f_j } \leq \gamma,
	\end{align*}
then
	\begin{align*}
	\max_j \E_j[ d(\hat{\theta},\theta_j)] 
	\geq
	\frac{\tau}{2} \left( 1 - \frac{\gamma+\log 2}{\log r} \right),
	\end{align*}
where $\E_j[\cdot]$ denotes the expectation with respect to distribution $f_j$.
\end{theorem}

To use this inequality, we will construct a set of densities on the set of covariance matrices corresponding distribution of the random matrix in the \textsf{MOD-SULQ} algorithm under different inputs.  These inputs will be chosen using a set of unit vectors which are a packing on the surface of the unit sphere.

\begin{lemma}
\label{lem:gaussKL}
Let $\Sigma$ be a positive definite matrix and let $f$ denote the density $\mathcal{N}(a, \Sigma)$ and $g$ denote the density $\mathcal{N}(b,\Sigma)$.  Then $\kldiv{f}{g} = \frac{1}{2} (a - b)^T \Sigma^{-1} (a - b)$.
\end{lemma}\begin{proof}
This is a simple calculation:
	\begin{align*}
	\kldiv{f}{g} &= \E_{x \sim f}\left[ -\frac{1}{2} (x - a)^T \Sigma^{-1} (x - a) + \frac{1}{2} (x - b) \Sigma^{-1} (x - b) \right] \\
	&= \frac{1}{2} \left( a^T \Sigma^{-1} a - a^T \Sigma^{-1} b - b^T \Sigma^{-1} a + b^T \Sigma^{-1} b \right) \\
	&= \frac{1}{2} (a - b)^T \Sigma^{-1} (a - b).
	\end{align*}
\end{proof}

The next theorem is a lower bound on the expected distance between the vector output by \textsf{MOD-SULQ} and the true top eigenvector.  In order to get this lower bound, we construct a class of data sets and use Theorem \ref{thm:fano} to derive a bound on the minimax error over the class.

\begin{theorem}[Utility bound for \modsulq]
\label{thm:sulq:lb}
Let $d$, $n$, and $\priveps > 0$ be given and let $\beta$ be given by Algorithm \ref{alg:modified_sulq} so that the output of \modsulq\ is $(\priveps,\delta)$-differentially private for all data sets in $\R^d$ with $n$ elements.  Then there exists a data set with $n$ elements such that if $\hat{v}_1$ denotes the output of \modsulq\ and $v_1$ is the top eigenvector of the empirical covariance matrix of the data set, the expected correlation $\ip{\hat{v}_1}{v_1}$ is upper bounded:
	\begin{align}
	\E \left[ |\ip{\hat{v}_1}{v_1}| \right] 
	&\le \min_{\phi \in \Phi} \left( 1 - \frac{ (1 - \phi) }{4} \left( 1 - \frac{ 1/\beta^2 + \log 2 }{ (d-1) \log \frac{1}{\sqrt{1 - \phi^2}} - \log(8)} \right)^2 \right),
	\label{eq:sulq:utilbnd}
	\end{align}
where 
	\begin{align}
	\Phi \in 
	\left[ \left. 
	\max\left\{ \frac{1}{\sqrt{2 \pi d}}, 
		\sqrt{1 - \exp\left( - \frac{2 \log(8d)}{d-1} \right)}, 
		\sqrt{ 1 - \exp\left( - \frac{2/\beta^2 + \log(256)}{d-1} \right) }
		\right\}, 
	1 \right) \right. .
	\label{eq:phibounds}
	\end{align}
\end{theorem}
\begin{proof}
For $\phi \in [ (2 \pi d)^{-1/2}, 1)$, Lemma \ref{lem:simplepack} shows there exists a set of $K$ unit vectors $\mc{C}$ such that for $\mu, \nu \in \mc{C}$, the inner product between them satisfies $\left| \ip{ \mu }{ \nu } \right| < \phi$, where $K$ is given by \eqref{eq:simplepack}.  Note that for small $\phi$ this setting of $K$ is loose, but any orthonormal basis provides $d$ unit vectors which are orthogonal, setting $K = d$ and solving for $\phi$ yields
	\begin{align*}
	\left( 1 - \exp\left( - \frac{ 2 \log(8d) }{ d-1 } \right) \right)^{1/2}.
	\end{align*}
Setting the lower bound on $\phi$ to the maximum of these two yields the set of $\phi$ and $K$ which we will consider in \eqref{eq:phibounds}.

For any unit vector $\mu$, let
	\begin{align*}
	A(\mu) = \mu \mu^T + N,
	\end{align*}
where $N$ is a $d \times d$ symmetric random matrix such that $\{ N_{ij} : 1 \le i \le j \le d \}$ are i.i.d. $\mc{N}(0,\beta^2)$, where $\beta^2$ is the noise variance used in the \modsulq\ algorithm.  Due to symmetry, the matrix $A(\mu)$ can be thought of as a jointly Gaussian random vector on the $d (d + 1)/2$ variables $\{A_{ij}(\mu) : 1 \le i \le j \le d\}$.  The mean of this vector is 
	\begin{align*}
	\bar{\mu} = \left( \mu_1^2, \mu_2^2, \ldots, \mu_{d}^2, \mu_1 \mu_2, \mu_1 \mu_3, \ldots, \mu_{d-1} \mu_d \right)^T,
	\end{align*}
and the covariance is $\beta^2 I_{d (d+1)/2}$.  Let $f_{\mu}$ denote the density of this vector.

For $\mu,\nu \in \mc{C}$, the divergence between $f_{\mu}$ and $f_{\nu}$ can be calculated using Lemma \ref{lem:gaussKL}:
	\begin{align}
	\kldiv{ f_{\mu} }{ f_{\nu} } &= \frac{1}{2} (\bar{\mu} - \bar{\nu})^T \Sigma^{-1} (\bar{\mu} - \bar{\nu}) \nonumber \\
	&= \frac{1}{2 \beta^2} \norm{ \bar{\mu} - \bar{\nu} }^2 \nonumber \\
	&\le \frac{1}{\beta^2}. \label{eq:packdiv}
	\end{align}
The last line follows from the fact that the vectors in $\mc{C}$ are unit norm.

For any two vectors $\mu,\nu \in \mc{C}$, lower bound the Euclidean distance between them using the upper bound on the inner product:
	\begin{align}
	\norm{\mu - \nu} \ge \sqrt{ 2 ( 1 - \phi) }.
	\label{eq:packdist}
	\end{align}

Let $\Theta = \Sp^{d-1}$ with the Euclidean norm and $\mc{R}$ be the set of distributions $\{A(\mu) : \mu \in \Theta\}$.  From \eqref{eq:packdist} and \eqref{eq:packdiv}, the set $\mc{C}$ satisfies the conditions of  Theorem \ref{thm:fano} with $\mc{F} = \{f_{\mu} : \mu \in \mc{C}\}$, $r = K$, $\tau = \sqrt{2 (1 - \phi)}$, and $\gamma = \frac{1}{\beta^2}$.  The conclusion of the Theorem shows that for \modsulq,
	\begin{align}
	\max_{\mu \in \mc{C}} \E_{f_{\mu}}\left[  \norm{ \hat{v} - \mu } \right]
	&\ge 
	\frac{ \sqrt{2 (1 - \phi)} }{2} \left( 1 - \frac{ 1/\beta^2 + \log 2 }{\log K} \right).
	\label{eq:fanonorm}
	\end{align}
This lower bound is vacuous when the term inside the parenthesis is negative, which imposes further conditions on $\phi$.  Setting $\log K = 1/\beta^2 + \log 2$, we can solve to find another lower bound on $\phi$:
	\begin{align*}
	\phi \ge \sqrt{ 1 - \exp\left( - \frac{2/\beta^2 + \log(256)}{d-1} \right) }.
	\end{align*}	
This yields the third term in \eqref{eq:phibounds}.  Note that for larger $n$ this term will dominate the others.
	
Using Jensen's inequality on the the left side of \eqref{eq:fanonorm}:
	\begin{align*}
	\max_{\mu \in \mc{C}} \E_{f_{\mu}}\left[  2 (1 - |\ip{\hat{v}}{\mu}|) \right]
	\ge \frac{ (1 - \phi) }{2} \left( 1 - \frac{ 1/\beta^2 + \log 2 }{\log K} \right)^2.
	\end{align*}
So there exists a $\mu \in \mc{C}$ such that 
	\begin{align}
	\E_{f_{\mu}}\left[ |\ip{\hat{v}}{\mu}| \right] 
	&\le 1 - \frac{ (1 - \phi) }{4} \left( 1 - \frac{ 1/\beta^2 + \log 2 }{\log K} \right)^2.
	\label{eq:fanoip}
	\end{align}
Consider the data set consisting of $n$ copies of $\mu$.  The corresponding covariance matrix is $\mu \mu^T$ with top eigenvector $v_1 = \mu$.  The output of the algorithm \textsf{MOD-SULQ} applied to this data set is an estimator of $\mu$ and hence satisfies \eqref{eq:fanoip}.  Minimizing over $\phi$ gives the desired bound.
\end{proof}

\begin{figure}[th]
\centering
\includegraphics[width=5in]{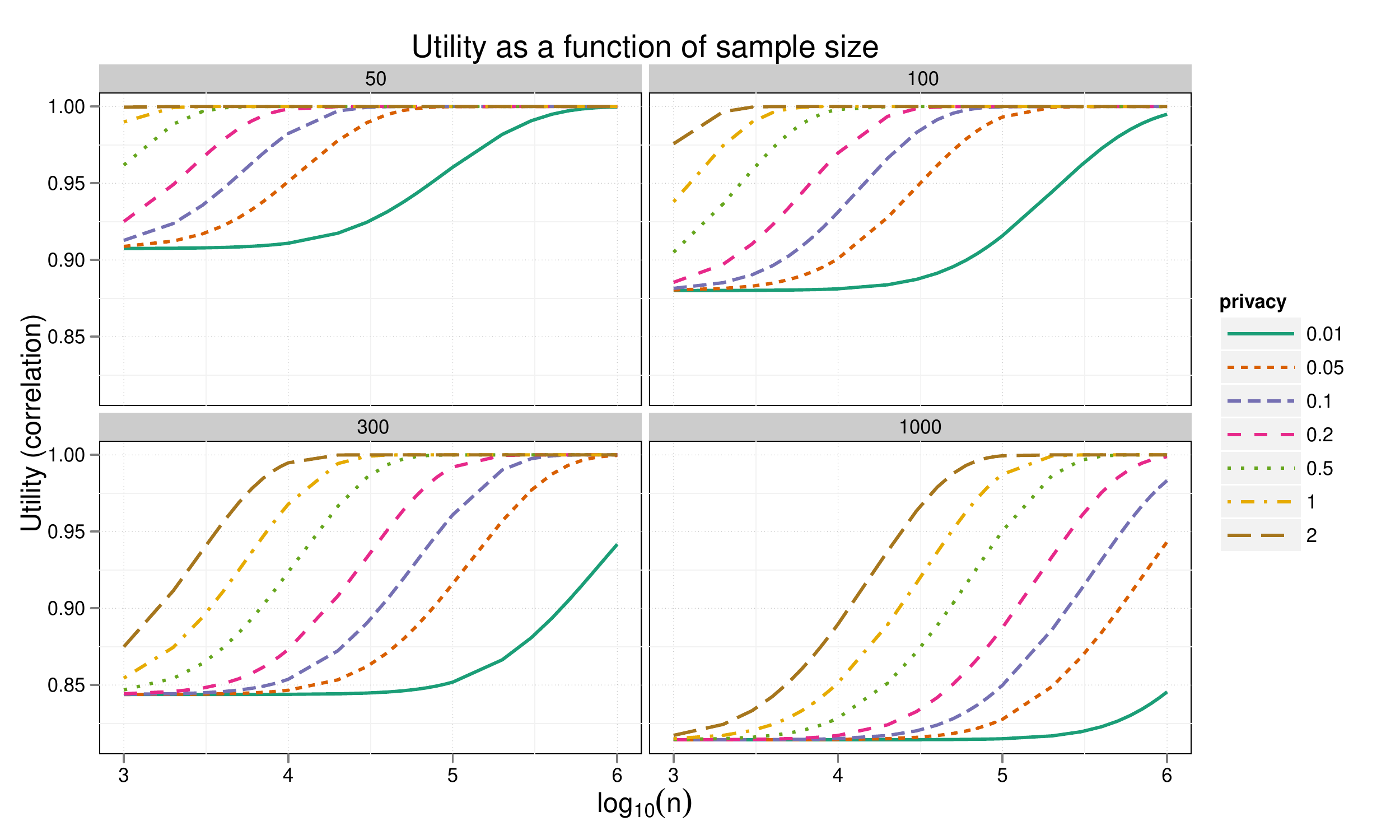}
\caption{ \small Upper bound from Theorem \ref{thm:sulq:lb} on the expected correlation between the true top eigenvector and the $\hat{v}_1$ produced by \modsulq.  The horizontal axis is $\log_{10}(n)$ and the vertical axis shows the lower bound in \eqref{eq:sulq:utilbnd}. The four panels correspond to different values of the dimension $d$, from $50$ to $1000$.  Each panel contains plots of the bound for different values of $\priveps$.
\label{fig:sulqLBpts}}
\end{figure}

The minimization over $\phi$ in \eqref{eq:sulq:utilbnd} does not lead to analytically pretty results, so we plotted the results in Figure \ref{fig:sulqLBpts} in order to get a sense of the bounds.  Figure \ref{fig:sulqLBpts} shows the lower bound on the expected correlation $\expe{|\ip{\hat{v}_1}{v_1}|}$ as a function of the number of data points (given on a logarithmic scale).  Each panel shows a different dimension, from $d = 50$ to $d = 1000$, and plots are given for different values of $\priveps$ ranging from $0.01$ to $2$.  In all experiments we set $\delta = 0.01$.  In high dimension, the lower bound shows that the expected performance of \modsulq\ is poor when there are a small number of data points.  This limitation may be particularly acute when the data lies in a very low dimensional subspace but is presented in very high dimension.  In such ``sparse'' settings, perturbing the input as in \modsulq\ is not a good approach.  However, in lower dimensions and data-rich regimes, the performance may be more favorable.

A little calculation yields the sample complexity bound in Theorem \ref{thm:lowersulq}

\begin{proof}
Suppose $\E\left[ |\ip{\hat{v_1}}{v_1}| \right] = \rho$.  Then a little algebra shows
	\begin{align*}
	2 \sqrt{1 - \rho} \ge \min_{\phi \in \Phi} \sqrt{1 - \phi} 
		\left(1 - \frac{ 1/\beta^2 + \log 2 }{ (d-1) \log \frac{1}{\sqrt{1 - \phi^2}} - \log(8)} \right).
	\end{align*}
Setting $\phi$ such that $(d-1) \log \frac{1}{\sqrt{1 - \phi^2}} - \log(8) = 2(1/\beta^2 + \log 2)$
we have
	\begin{align*}
	4 \sqrt{1 - \rho} \ge \sqrt{1 - \phi}.
	\end{align*}
Since we are concerned with the scaling behavior for large $d$ and $n$, this implies
	\begin{align*}
	\log \frac{1}{\sqrt{1 - \phi^2}} = \Theta\left( \frac{1}{\beta^2 d} \right),
	\end{align*}
so
	\begin{align*}
	\phi &= \sqrt{ 1 -  \exp\left( - \Theta\left( \frac{1}{\beta^2 d} \right) \right) } \\
	&= \Theta\left( \sqrt{ \frac{1}{\beta^2 d} } \right).
	\end{align*}
From Algorithm \ref{alg:modified_sulq}, to get for some constant $c_1$, we have the following lower bound on $\beta$:
	\begin{align*}
	\beta^2 > c_1 \frac{d^2}{n^2 \priveps^2} \log(d/\delta).
	\end{align*}
Substituting, we get for some constants $c_2$ and $c_3$ that 
	\begin{align*}
	\left( 1 - c_2 (1 - \rho) \right) \le c_3 \frac{ n^2 \priveps^2 }{d^3 \log (d/\delta)}.
	\end{align*}
Now solving for $n$ shows
	\begin{align*}
	n \ge c \frac{ d^{3/2} \sqrt{\log(d/\delta)} }{ \priveps} \left( 1 - c' (1 - \rho) \right).
	\end{align*}
\end{proof}

\section{Experiments}

We next turn to validating our theoretical results on real data. We implemented \modsulq\ and \ouralgorithm\ in order to test our theoretical bounds.  Implementing \ouralgorithm\ involved using a Gibbs sampling procedure \citep{Hoff:09bingham}.  A crucial parameter in MCMC procedures is the burn-in time, which is how long the chain must be run for it to reach its stationary distribution.  Theoretically, chains reach their stationary distribution only in the limit; however, in practice MCMC users must sample after some finite time.  In order to use this procedure appropriately, we determined a burn-in time using our data sets.  The interaction of MCMC procedures and differential privacy is a rich area for future research.

\subsection{Data and preprocessing}

We report on the performance of our algorithm on some real datasets. We chose four datasets from four different domains---\texttt{kddcup99}~\citep{uciadult}, which includes features of $494{,}021$ network connections, \texttt{census}~\citep{uciadult}, a demographic data set on $199{,}523$ individuals, \texttt{localization}~\citep{localization}, a medical dataset with $164{,}860$ instances of sensor readings on individuals engaged in different activities, and \texttt{insurance}~\citep{insurance}, a dataset on product usage and demographics of $9{,}822$ individuals.  

These datasets contain a mix of continuous and categorical features. We preprocessed each dataset by converting a feature with $q$ discrete values to a vector in $\{0, 1\}^q$; after preprocessing, the datasets \texttt{kddcup99}, \texttt{census}, \texttt{localization} and \texttt{insurance} have dimensions $116$, $513$, $44$ and $150$ respectively. We also normalized each row so that each entry has maximum value $1$, and normalize each column such that the maximum (Euclidean) column norm is $1$.  We choose $k = 4$ for \texttt{kddcup}, $k=8$ for \texttt{census}, $k = 10$ for \texttt{localization} and $k=11$ for \texttt{insurance}; in each case, the utility $\qnorm(V_k)$ of the top-$k$ PCA subspace of the data matrix accounts for at least $80\%$ of $\frobnorm{A}$. Thus, all four datasets, although fairly high dimensional, have good low-dimensional representations. The properties of each dataset are summarized in Table~\ref{table:datasetsummary}. 

\begin{table*}\label{table:datasetsummary}
\centering
\begin{tabular}{l|c|c|c|c|c}
\textbf{Dataset} & \textbf{\#instances} & \textbf{\#dimensions} & $k$ & $\qnorm(V_k)$ & $\qnorm(V_k)/\frobnorm{A}$ \\ \hline \hline
\texttt{kddcup} & 494,021 & 116 & 4 & 0.6587 & 0.96 \\ \hline
\texttt{census} & 199,523 & 513 & 8 & 0.7321 & 0.81 \\ \hline
\texttt{localization} & 164,860 & 44 & 10 & 0.5672  & 0.81 \\ \hline
\texttt{insurance} & 9,822 & 150 & 11 & 0.5118  & 0.81 
\end{tabular}
\caption{Parameters of each dataset. The second column is the number of dimensions after preprocessing. $k$ is the dimensionality of the PCA, the third column contains $\qnorm(V_k)$, where $V_k$ is the top-$k$ PCA subspace, and the fifth column is the normalized utility $\qnorm(V_k)/\frobnorm{A}$.}
\end{table*}

\subsection{Implementation of Gibbs sampling \label{sec:gibbs}}

The theoretical analysis of \ouralgorithm\ uses properties of the Bingham distribution $\mathsf{BMF}_k(\cdot)$ given in \eqref{eq:bmf:density}.  To implement this algorithm for experiments we use a Gibbs sampler due to \citet{Hoff:09bingham}.  The Gibbs sampling scheme induces a Markov Chain, the stationary distribution of which is the density in \eqref{eq:bmf:density}.  Gibbs sampling and other MCMC procedures are widely used in statistics, scientific modeling, and machine learning to estimate properties of complex distributions \cite{Brooks1998}.

Finding the speed of convergence of MCMC methods is still an open area of research.  There has been much theoretical work on estimating convegence times~\citep{JonesH:04burnin,Douc2004,Jones2001,Roberts1999,Roberts2001,Roberts1999,Roberts2001,Rosenthal1995,Kolassa1999,Kolassa2000}, but unfortunately, most theoretical guarantees are available only in special cases and are often too weak for practical use.  In lieu of theoretical guarantees, users of MCMC methods empirically estimate the {\em burn-in time}, or the number of iterations after which the chain is sufficiently close to its stationary distribution.  Statisticians employ a range of diagnostic methods and statistical tests to empirically determine if the Markov chain is close to stationarity \citep{Cowles1996,Brooks1998a,Brooks1998b,Eladlouni2006}.  These tests do not provide a sufficient guarantee of stationarity, and there is no ``best test'' to use.  In practice, the convergence of derived statistics is used to estimate an appropriate the burn-in time.  In the case of the Bingham distribution, \citet{Hoff:09bingham} performs qualitative measures of convergence.  Developing a better characterization of the convergence of this Gibbs sampler is an important question for future work.

Because the MCMC procedure of \citet{Hoff:09bingham} does not come with convergence-time guarantees, for our experiments we had to choose an appropriate burn-in time.  The 	``ideal'' execution of \ouralgorithm\ provides $\priveps$-differential privacy, but because our implementation only approximates sampling from the Bingham distribution, we cannot guarantee that this implementation provides the privacy guarantee.  As noted by \citet{mironov-CCS12}, even current implementations of floating-point arithmetic may suffer from privacy problems, so there is still significant work to do between theory and implementation.  For this paper we tried to find a burn-in time that was sufficiently long so that we could be confident that the empirical performance of \ouralgorithm\ was not affected by the initial conditions of the sampler.

In order to choose an appropriate burn-in time, we examined the {\em{time series trace}} of the Markov Chain. We ran $l$ copies, or traces, of the chain, starting from $l$ different initial locations drawn uniformly from the set of all $d \times k$ matrices with orthonormal columns.  Let $X^{i}(t)$ be the output of the $i$-th copy at iteration $t$, and let $U$ be the top-$k$ PCA subspace of the data. We used the following statistic as a function of iteration $T$:
\begin{align*}
F_k^{i}(T) = \frac{1}{\sqrt{k}} \norm{\frac{1}{T} \sum_{t=1}^{T} X^{i}(t)}_F, 
\end{align*}
where $||\cdot||_F$ is the Frobenius norm. The matrix Bingham distribution has mean $0$, and hence with increasing $T$, the statistic $F_k^{i}(T)$ should converge to $0$. 

\begin{figure}[ht]
\centering
\subfigure[\texttt{kddcup} ($k = 4$)]{
\includegraphics[width=2.8in]{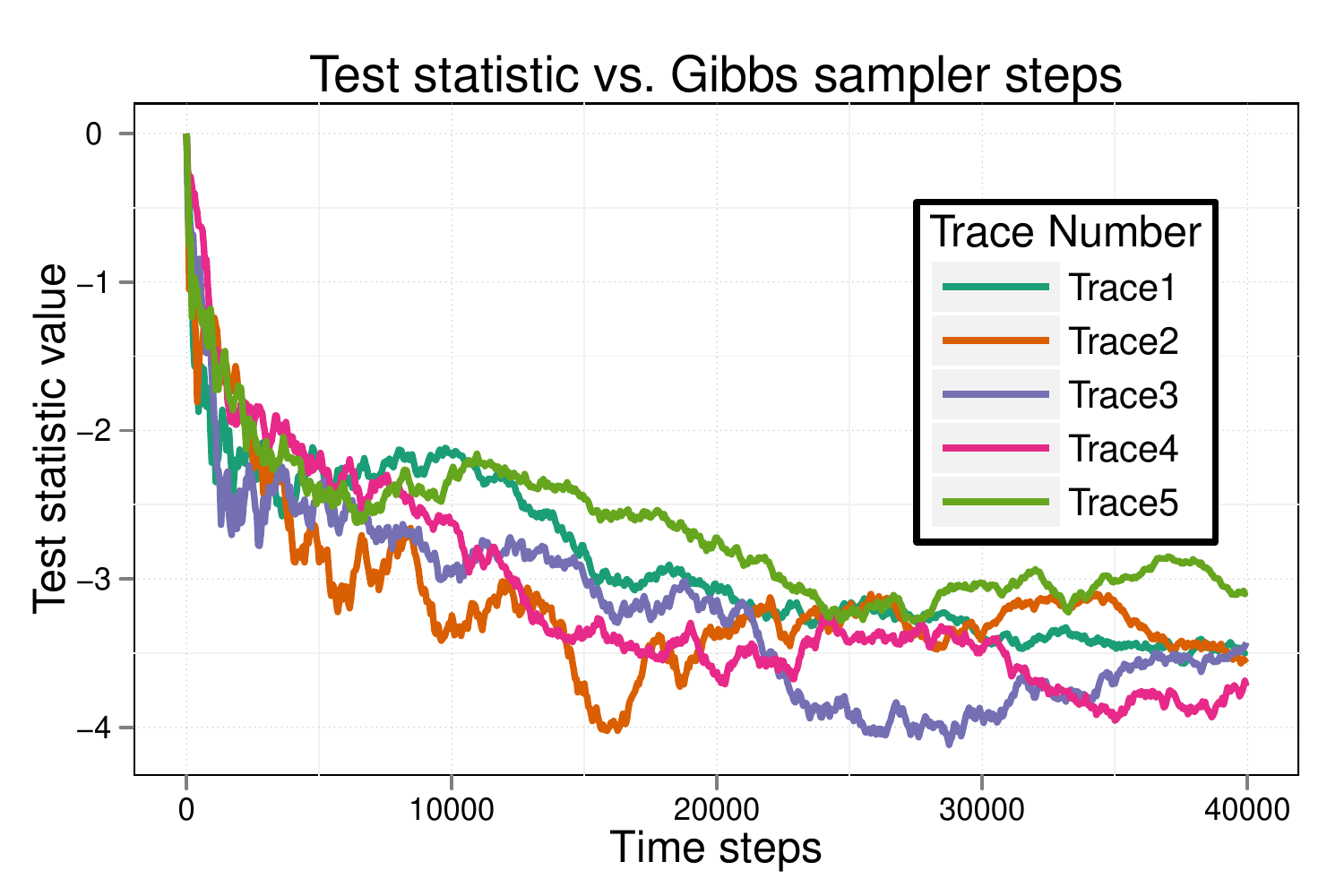}
\label{fig:Fkddcup}
}
\subfigure[\texttt{insurance} ($k = 11$)]{
\includegraphics[width=2.8in]{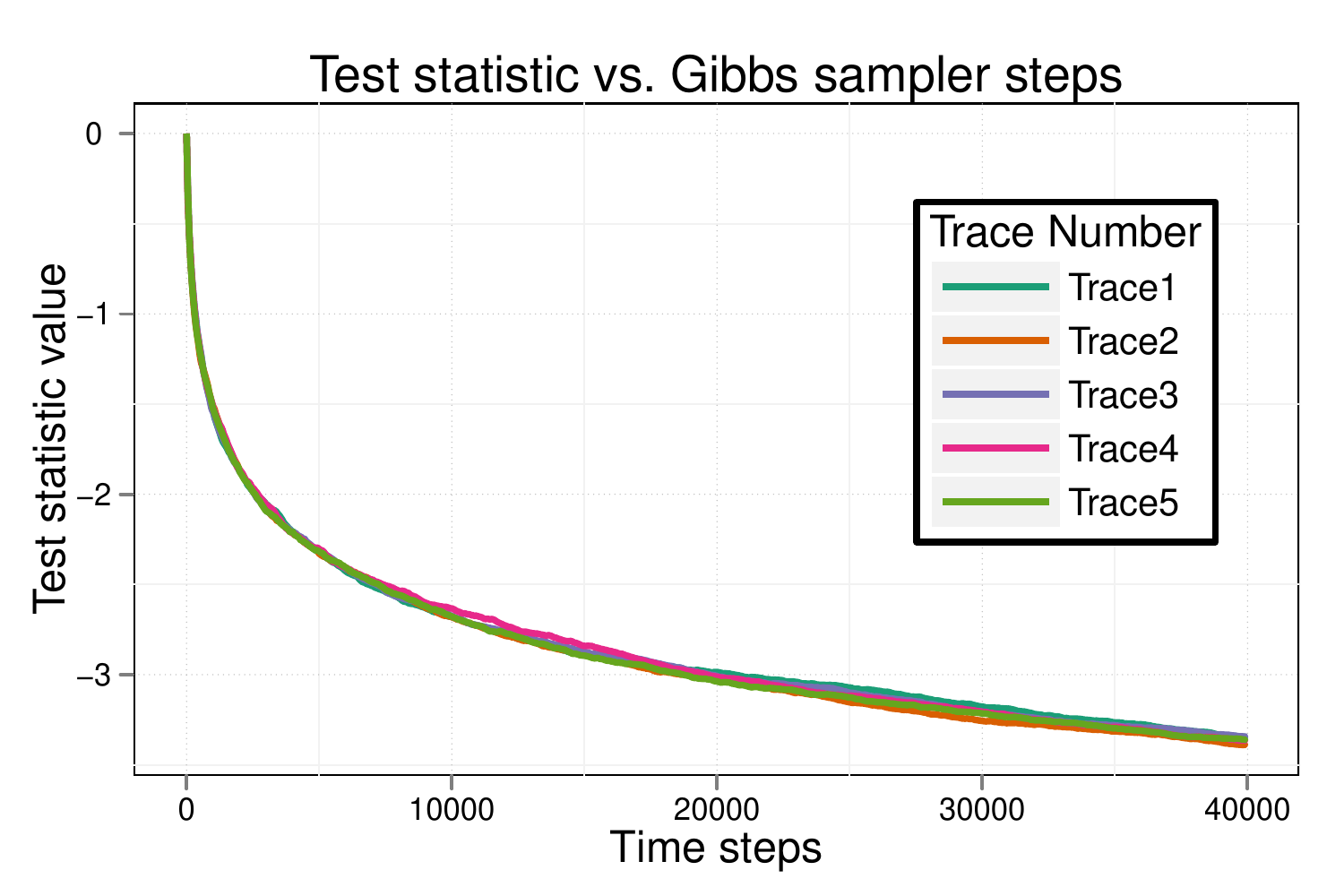}
\label{fig:Finsurance}
}
\caption{\small Plots of $\log F_k^{i}(T)$ for five different traces (values of $i$) on two different data sets.  Figure \ref{fig:Fkddcup} shows $\log F_k^{i}(T)$ for for $k=4$ as a function of iteration $T$ for $40,000$ steps of the Gibbs sampler on the \texttt{kddcup} dataset.  Figure \ref{fig:Finsurance} shows the same for the \texttt{insurance} dataset. \label{fig:MCMCtrace}}
\end{figure}

Figure \ref{fig:MCMCtrace} illustrates the behavior of the Gibbs sampler.  The plots show the value of $\log F_k^{i}(T)$ as a function of the Markov chain iteration for 5 different restarts of the MCMC procedure for two data sets, \texttt{kddcup} and \texttt{insurance}.  The initial starting points were chosen uniformly from the set of all $d \times k$ matrices with orthonormal columns.
The plots show that $F_k^{i}(T)$ decreases rapidly after a few thousand iterations, and is less than $0.01$ after $T = 20{,}000$ in both cases. $\log F_k^{i}(T)$ also appears to have a larger variance for \texttt{kddcup} than for \texttt{insurance}; this is explained by the fact that the \texttt{kddcup} dataset has a much larger number of samples, which makes its stationary distribution farther from the initial distribution of the
sampler.
Based on these and other simulations, we observed that the Gibbs sampler converges to $F_k(t) < 0.01$ at $t = 20{,}000$ when run on data with a few hundred dimensions and with $k$ between $5$ and $10$; we thus chose to run the Gibbs sampler for $T = 20{,}000$ timesteps for all the datasets. 

Our simulations indicate that the chains converge fairly rapidly, particularly when $\frobnorm{A - A_k}$ is small so that $A_k$ is a good approximation to $A$.  Convergence is slower for larger $n$ when the initial state is chosen from the uniform distribution over all $k \times d$ matrices with orthonormal columns; this is explained by the fact that for larger $n$, the stationary distribution is farther in variation distance from the starting distribution, which results in a longer convergence time.

\subsection{Scaling with data set size}

We ran three algorithms on these data sets : standard (non-private) PCA, \modsulq, and \ouralgorithm.  As a sanity check, we also tried a uniformly generated random projection---since this projection is data-independent we would expect it to have low utility.  We measured the utility $\qnorm(U)$, where $U$ is the $k$-dimensional subspace output by the algorithm; $\qnorm(U)$ is maximized when $U$ is the top-$k$ PCA subspace, and thus this reflects how close the output subspace is to the true PCA subspace in terms of representing the data.  Although our theoretical results hold for $\qangle(\cdot)$, the ``energy'' $\qnorm(\cdot)$ is more relevant in practice for larger $k$.

To investigate how well these different algorithms performed on real data, for each data set we subsampled  data sets of different sizes $n$ uniformly and ran the algorithms on the subsets.  We chose $\priveps = 0.1$ for this experiment, and for \modsulq\ we used $\delta=0.01$.  We averaged over $5$ such subsets and over several instances of the randomized algorithms ($10$ restarts for \ouralgorithm\ and $100$ for \modsulq\ and random projections).  For each subset and instance we calculated the resulting utility $\qnorm(\cdot)$ of the output subspace.  

Figures \ref{fig:utilitycensus}, \ref{fig:utilitykddcup}, \ref{fig:utilitylocalization}, and \ref{fig:utilityinsurance} show $\qnorm(U)$ as a function of the subsampled data set sizes.  The bars indicate the standard deviation over the restarts (from subsampling the data and random sampling for privacy).  The non-private algorithm achieved $\qnorm(V_k)$ for nearly all subset sizes (see Table \ref{table:datasetsummary} for the values).  These plots illustrate how additional data can improve the utility of the output for a fixed privacy level $\priveps$. As $n$ increases, the dashed blue line indicating the utility of \ouralgorithm\ begins to approach $\qnorm(V_k)$, the utility of the optimal subspace. 

These experiments also show that the performance of \ouralgorithm\ is significantly better than that of \modsulq, and \modsulq\ produces subspaces whose utility is on par with randomly choosing a subspace.
The only exception to this latter point is {\texttt{localization}},  We believe this is because $d$ is much lower for this data set ($d = 44$), which shows that for low dimension and large $n$, \modsulq\ may produce subspaces with reasonable utility.  Furthermore, \modsulq\ is simpler and hence runs faster than \ouralgorithm, which requires running the Gibbs sampler past the burn-in time.   Our theoretical results suggest that the performance of differentially private PCA cannot be significantly improved over the performance of \ouralgorithm\, but since those results hold for $k = 1$ they do not immediately apply here.

\begin{figure}[ht]
\centering
\subfigure[\texttt{census} ($k = 8$)]{
\includegraphics[scale=0.38]{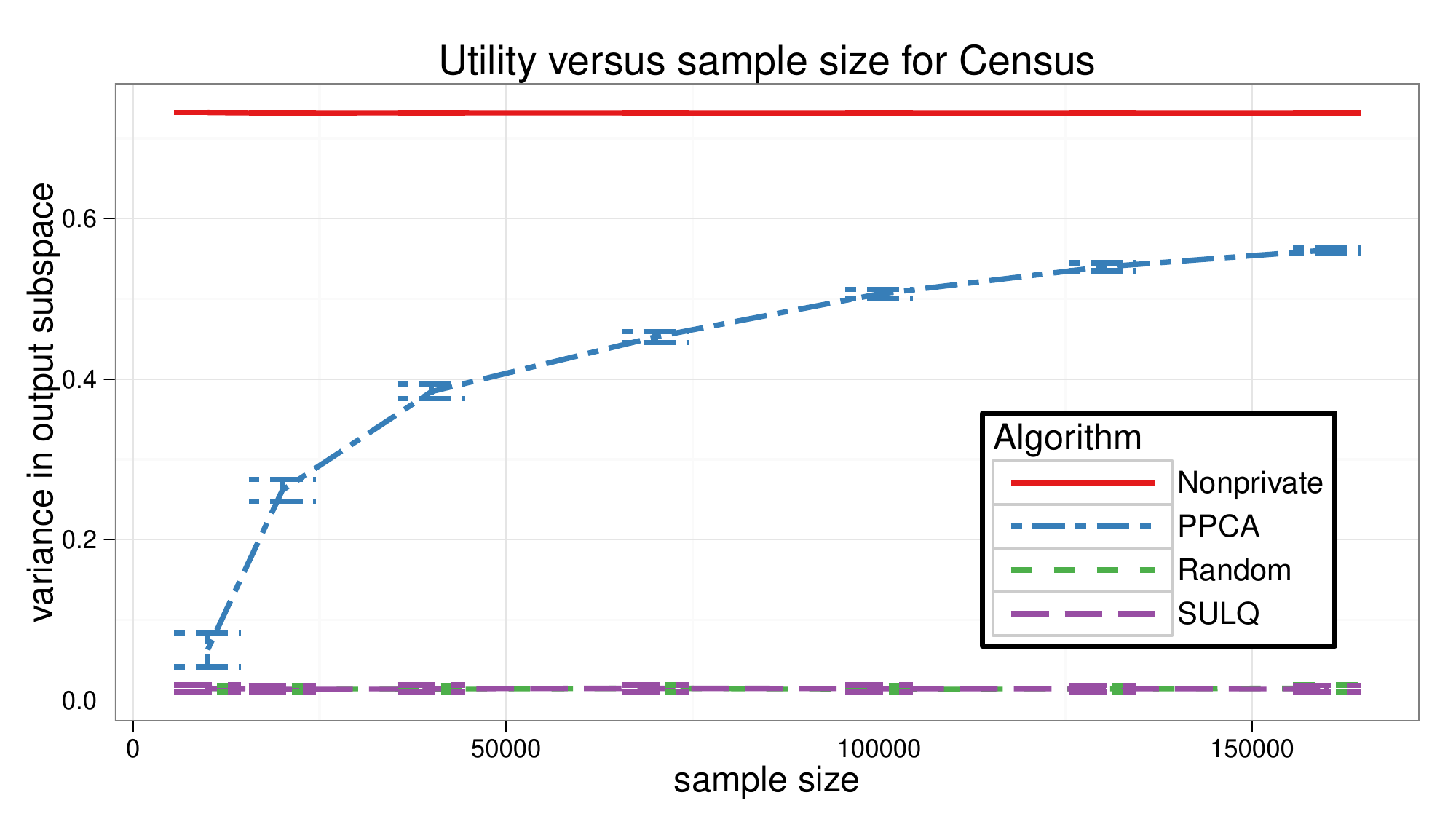}
\label{fig:utilitycensus}
}%
\subfigure[\texttt{kddcup} ($k = 4$)]{
\includegraphics[scale=0.38]{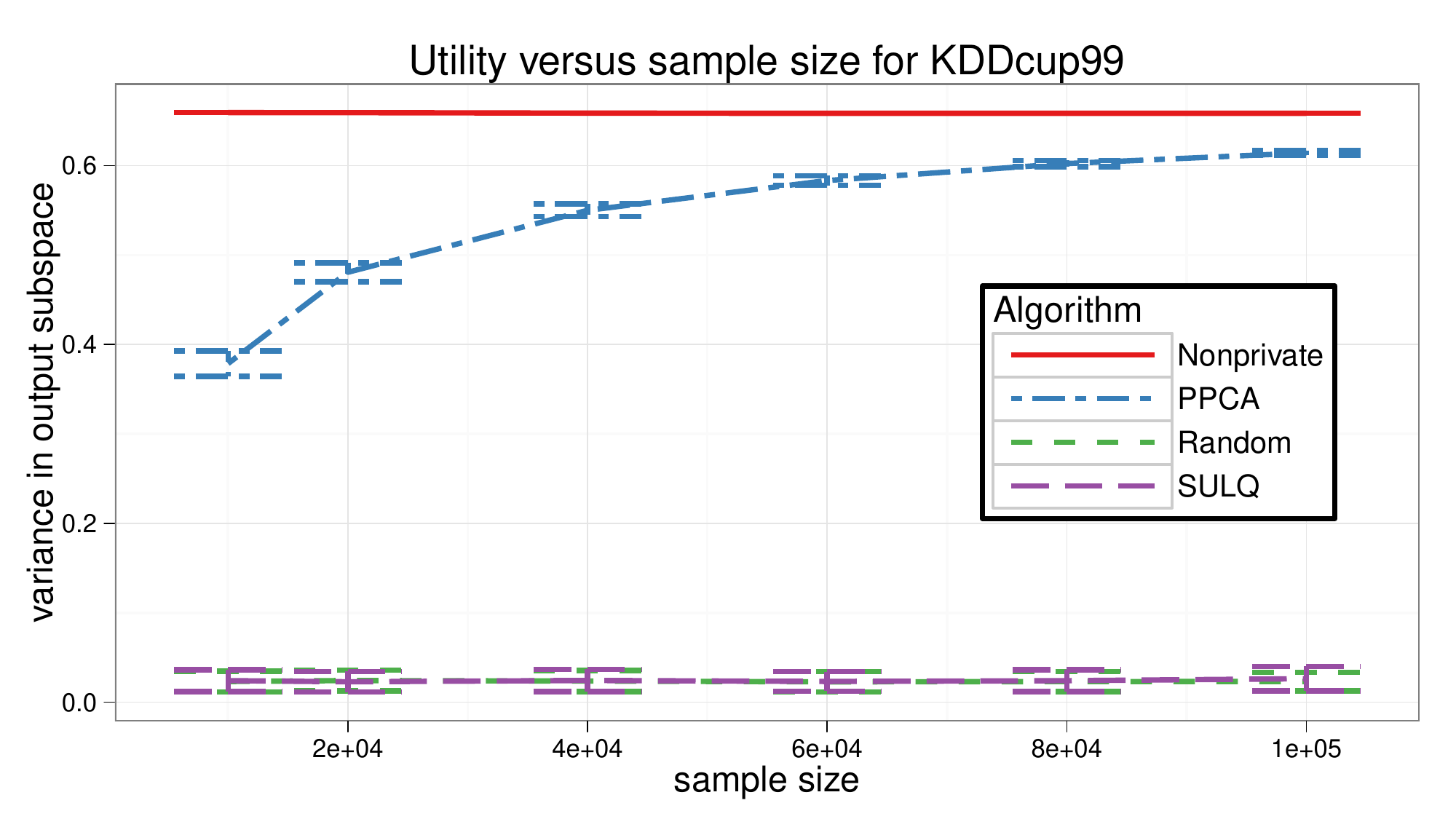}
\label{fig:utilitykddcup}
}
\caption{\small Plot of the unnormalized utility $\qnorm(U)$ versus the sample size $n$, averaged over random subsets of the data and randomness in the algorithms.  The bars are at one standard deviation about the mean.  The top red line is the PCA algorithm without privacy constraints.  The dashed line in blue is the utility for \ouralgorithm.  The green and purple dashed lines are nearly indistinguishable and represent the utility from  random projections and \modsulq, respectively.  In these plots $\priveps = 0.1$ and $\delta = 0.01$.  \label{fig:utilityplots} }
\end{figure}

\begin{figure}[ht]
\centering
\subfigure[\texttt{localization} ($k = 10$)]{
\includegraphics[scale=0.38]{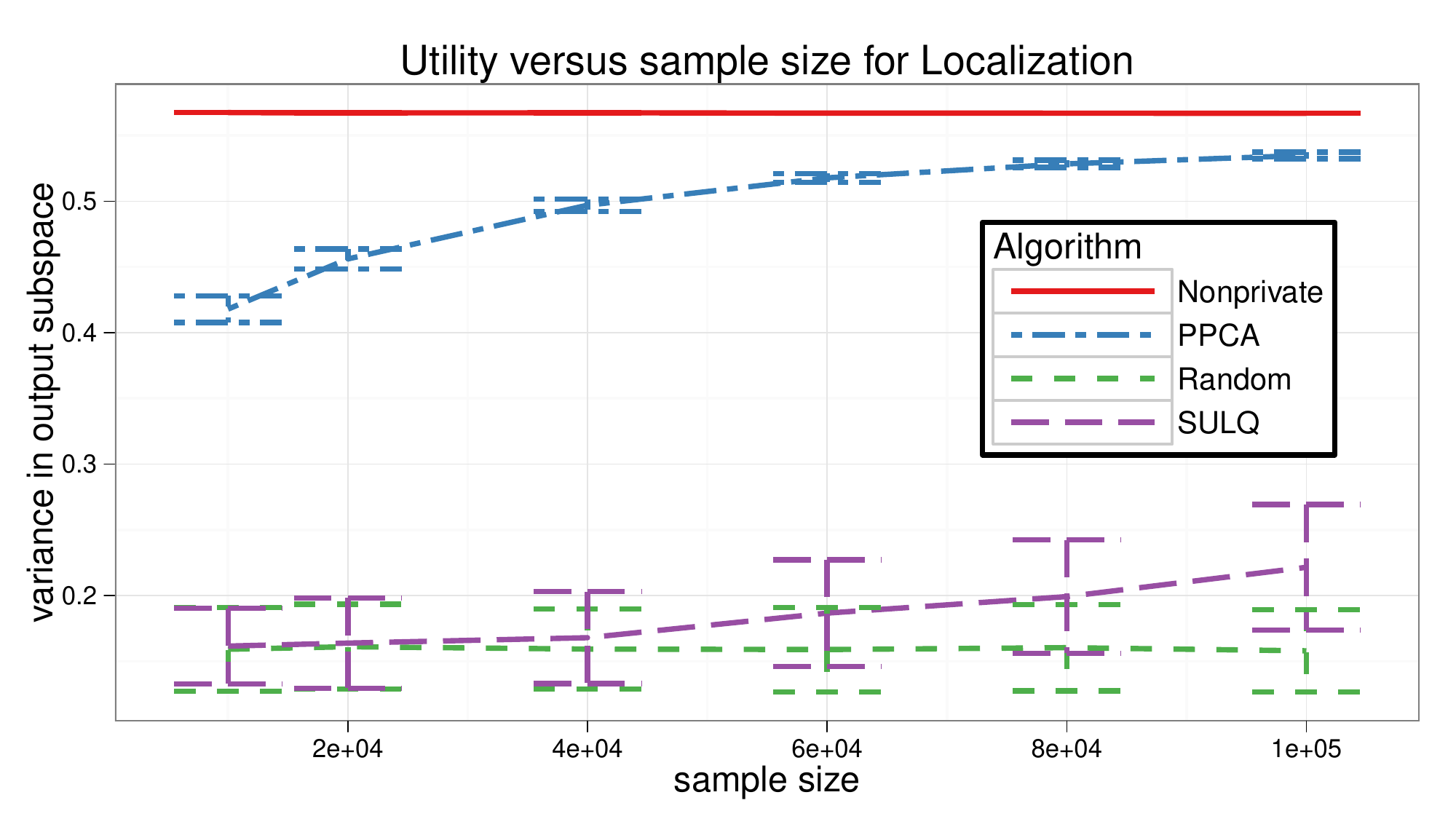}
\label{fig:utilitylocalization}
}%
\subfigure[\texttt{insurance} ($k = 11$)]{
\includegraphics[scale=0.38]{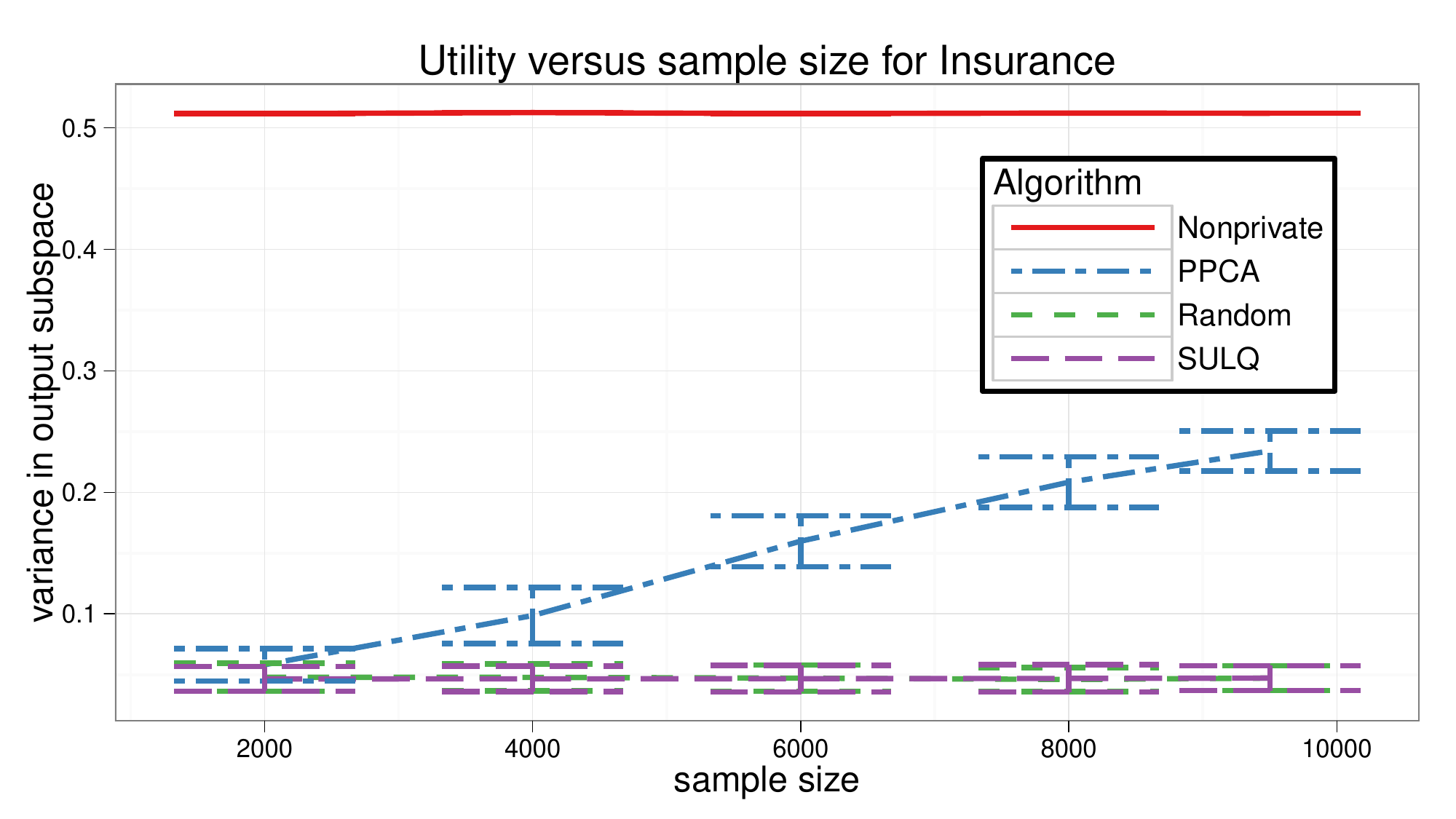}
\label{fig:utilityinsurance}
}
\caption{\small Plot of the unnormalized utility $\qnorm(U)$ versus the sample size $n$, averaged over random subsets of the data and randomness in the algorithms.  The bars are at one standard deviation about the mean.  The top red line is the PCA algorithm without privacy constraints.  The dashed line in blue is the utility for \ouralgorithm.  The green and purple dashed lines are nearly indistinguishable for \texttt{insurance} but diverge for \texttt{localization}---they represent the utility from random projections and \modsulq, respectively.  In these plots $\priveps = 0.1$ and $\delta = 0.01$.   \label{fig:utilityplots2} }
\end{figure}

\subsection{Effect of privacy on classification}

 A common use of a dimension reduction algorithm is as a precursor to classification or clustering; to evaluate the effectiveness of the different algorithms, we projected the data onto the subspace output by the algorithms, and measured the classification accuracy using the projected data.  The classification results are summarized in Table~\ref{table:classification}. We chose the {\em{normal}} vs. all classification task in \texttt{kddcup99}, and the {\em{falling}} vs. all classification task in \texttt{localization}. \footnote{For the other two datasets, \texttt{census} and \texttt{insurance}, the classification accuracy of linear SVM after (non-private) PCAs is as low as always predicting the majority label.}   We used a linear SVM for all classification tasks, which is implemented by libSVM~\citep{CC01a}. 

For the classification experiments, we used half of the data as a holdout set for computing a projection subspace.
 We projected the classification data onto the subspace computed based on the holdout set; $10\%$ of this data was used for training and parameter-tuning, and the rest for testing. We repeated the classification process $5$ times for $5$ different (random) projections for each algorithm, and then ran the entire procedure 
over $5$ random permutations of the data. Each value in the figure is thus an average over $5 \times 5 = 25$ rounds of classification.

\begin{table*}\label{table:classification}
\centering
\begin{tabular}{l|c|c}
& \textbf{\texttt{kddcup99}} & \textbf{\texttt{localization}} \\ \hline \hline
Non-private PCA & $98.97 \pm 0.05$ & $100 \pm 0$ \\ \hline
\ouralgorithm & $98.95\pm 0.05$ & $100 \pm 0$ \\ \hline
\modsulq\ & $98.18\pm 0.65$& $97.06\pm 2.17$ \\ \hline
Random Projections & $98.23\pm 0.49$ & $96.28 \pm 2.34$ 
\end{tabular}
\caption{Classification accuracy in the $k$-dimensional subspaces for \texttt{kddcup99} ($k=4$), and \texttt{localization} ($k=10$) in the $k$-dimensional subspaces reported by the different algorithms.}
\end{table*}

The classification results show that our algorithm performs almost as well as non-private PCA for classification in the top-$k$ PCA subspace, while the performance of \modsulq\ and random projections are a little worse. The classification accuracy while using \modsulq\ and random projections also appears to have higher variance compared to our algorithm and non-private PCA.  This is because the projections tend to be farther from the top-$k$ PCA subspace, making the classification error more variable.

\subsection{Effect of the privacy requirement}

How to choose $\priveps$ is important open question for many applications.  We wanted to understand the impact of varying $\priveps$ on the utility of the subspace.  We did this via a synthetic data set---we generated $n = 5{,}000$ points drawn from a Gaussian distribution in $d = 10$ with mean $\mbf{0}$ and covariance matrix with eigenvalues 
\begin{align}
\{ 0.5, 0.30, 0.04, 0.03, 0.02, 0.01,0.004, 0.003, 0.001, 0.001 \}.
\label{eq:synthdata}
\end{align}
In this case the space spanned by the top two eigenvectors has most of the energy, so we chose $k = 2$ and plotted the utility $\qnorm(\cdot)$ for non-private PCA, \modsulq\ with $\delta = 0.05$, and \ouralgorithm\ with a burn-in time of $T = 1000$.  
We drew 100 samples from each privacy-preserving algorithm and the plot of the average utility versus $\priveps$ is shown in Figure \ref{fig:utilalpha}.  The privacy requirement relaxes as $\priveps$ increases, and both \modsulq\ and \ouralgorithm\ approach the utility of PCA without privacy constraints.  However, for moderate $\priveps$ \ouralgorithm\ still captures most of the utility, whereas the gap between \modsulq\ and \ouralgorithm\ becomes quite large.

\begin{figure}[ht]
\centering
\includegraphics[width=5in]{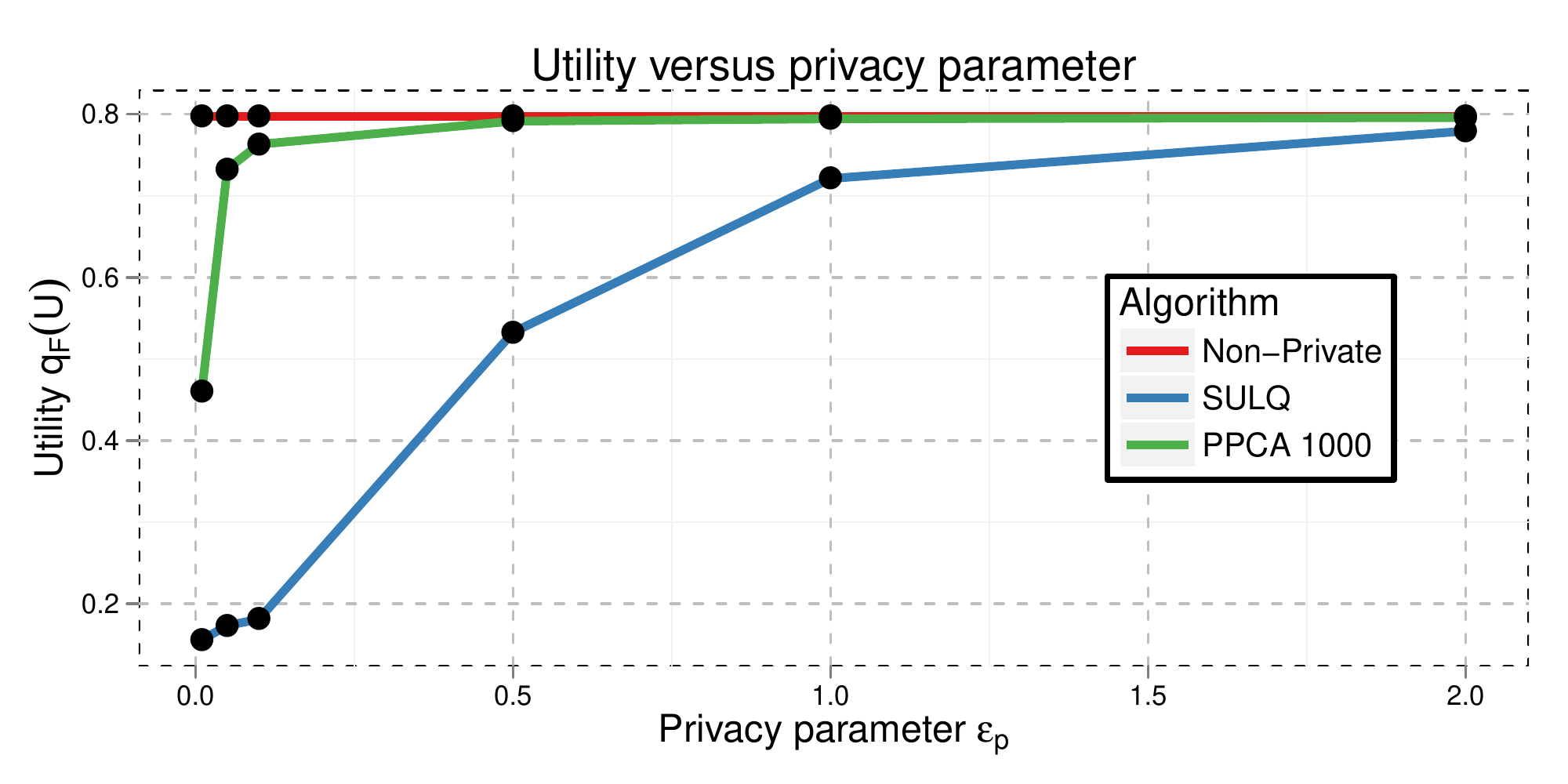}
\caption{Plot of $\qnorm(U)$ versus $\priveps$ for a synthetic data set with $n = 5{,}000$, $d = 10$, and $k = 2$.  The data has a Gaussian distribution whose covariance matrix has eigenvalues given by \eqref{eq:synthdata}.}
\label{fig:utilalpha}
\end{figure}

\section{Conclusion}

In this paper we investigated the theoretical and empirical performance of differentially private approximations to PCA.  Empirically, we showed that \modsulq\ and \ouralgorithm\ differ markedly in how well they approximate the top-$k$ subspace of the data.  The reason for this, theoretically, is that the sample complexity of \modsulq\ scales as $d^{3/2} \sqrt{\log d}$ whereas \ouralgorithm\ scales as $d$.  Because \ouralgorithm\ uses the exponential mechanism with $\qnorm(\cdot)$ as the utility function, it is not surprising that it performs well.  However, \modsulq\ often had a performance comparable to random projections, indicating that the real data sets had too few points for it to be effective.  We furthermore showed that \ouralgorithm\ is nearly optimal, in that any differentially private approximation to PCA must use $\Omega(d)$ samples.  

Our investigation brought up many interesting issues to consider for future work.  The description of differentially private algorithms assume an ideal model of computation : real systems require additional security assumptions that have to be verified.  The difference between truly random noise and pseudorandomness and the effects of finite precision can lead to a gap between the theoretical ideal and practice.  Numerical optimization methods used in some privacy methods~\citep{ChaudhuriMS:11erm} can only produce approximate solutions; they may also have complex termination conditions unaccounted for in the theoretical analysis.  MCMC sampling is similar : if we can guarantee that the sampler's distribution has total variation distance $\delta$ from the Bingham distribution, then sampler can guarantee $(\priveps,\delta)$ differential privacy.   However, we do not yet have such analytical bounds on the convergence rate; we must determine the Gibbs sampler's convergence empirically.  Accounting for these effects is an interesting avenue for future work that can bring theory and practice together.

For PCA more specifically, it would be interesting to extend the sample complexity results to general $k > 1$.  For $k = 1$ the utility functions $\qnorm(\cdot)$ and $\qangle(\cdot)$ are related, but for larger $k$ it is not immediately clear what metric best captures the idea of ``approximating'' the top-$k$ PCA subspace.   For minimax lower bounds, it may be possible to construct a packing with respect to a general utility metric.  For example, \citet{KT13} use properties of packings on the Grassmann manifold.  Upper bounds on the sample complexity for \ouralgorithm\ may be possible by performing a more careful analysis of the Bingham distribution or by finding better approximations for its normalizing constant.  Developing a framework for analyzing general approximations to PCA may be of interest more broadly in machine learning.

\section*{Acknowledgements}

The authors would like to thank the reviewers for their detailed comments, which greatly improved the quality and readability of the manuscript, and the action editor, Gabor Lugosi, for his patience during the revision process.  KC and KS would like to thank NIH for research support under U54-HL108460.
The experimental results were made possible by support from the UCSD
FWGrid Project, NSF Research Infrastructure Grant Number EIA-0303622.
ADS was supported in part by the California Institute for Telecommunications and
Information Technology (CALIT2) at UC San Diego.

\nocite{HLMJ09,CM06,MKAGV08,AS00,EGS03,S02,MGKV06,LiLV10}

\appendix

\section{A packing lemma}

The proof of this lemma is relatively straightforward.  The following is a slight refinement of a lemma due to~\cite{CsiszarN:88positivity,CsiszarN:91gavc}.

\begin{lemma}
\label{lem:conc2}
Let $\mbf{Z}_1, \mbf{Z}_2, \ldots, \mbf{Z}_N$ be arbitrary random variables and let $f_i(\mbf{Z}_1, \ldots, \mbf{Z}_i)$ be arbitrary with $0 \le f_i \le 1$, $i = 1, 2, \ldots, N$.  Then the condition
	\begin{align*}
	\expe{ f_i(\mbf{Z}_1, \ldots, \mbf{Z}_i) | \mbf{Z}_1,\ldots, \mbf{Z}_{i-1} }
		\le a_i\ \mathrm{a.s.}, \qquad i = 1, 2, \ldots, N
	\end{align*}
implies that
	\begin{align*}
	\prob{ \frac{1}{N} \sum_{i=1}^{N} f_i(\mbf{Z}_1, \ldots, \mbf{Z}_i) > t }
	\le
	\exp\left(-N t (\log 2) + \sum_{i=1}^{N} a_i \right).
	\end{align*}
\end{lemma}

\begin{proof}
First apply Markov's inequality:
	\begin{align*}
	\prob{ \frac{1}{N} \sum_{i=1}^{N} f_i(\mbf{Z}_1, \ldots, \mbf{Z}_i) > t } & \\
	&\hspace{-1.5in} = \prob{ 2^{ \sum_{i=1}^{N} f_i(\mbf{Z}_1, \ldots, \mbf{Z}_i) } > 2^{Nt} } \\
	&\hspace{-1.5in} \le 2^{-Nt} \expe{ 2^{ \sum_{i=1}^{N} f_i(\mbf{Z}_1, \ldots, \mbf{Z}_i) } } \\
	&\hspace{-1.5in} \le 2^{-Nt} \expe{ 2^{ \sum_{i=1}^{N-1} f_i(\mbf{Z}_1, \ldots, \mbf{Z}_i) }
		\expe{ 2^{ f_N(\mbf{Z}_1, \ldots, \mbf{Z}_N) } | \mbf{Z}_1, \ldots, \mbf{Z}_{N-1} }  }.
	\end{align*}
Now note that for $b \in [0,1]$ we have $2^b \le 1 + b \le e^{b}$, so
	\begin{align*}
	\expe{ 2^{ f_N(\mbf{Z}_1, \ldots, \mbf{Z}_N) } | \mbf{Z}_1, \ldots, \mbf{Z}_{N-1} }
	&\le \expe{ 1 + f_N(\mbf{Z}_1, \ldots, \mbf{Z}_N) | \mbf{Z}_1, \ldots, \mbf{Z}_{N-1} } \\
	&\le (1 + a_N) \\
	&\le \exp(a_N).
	\end{align*}
Therefore
	\begin{align*}
	\prob{ \frac{1}{N} \sum_{i=1}^{N} f_i(\mbf{Z}_1, \ldots, \mbf{Z}_i) > t }
	\le \exp(-Nt (\log 2) + a_N)
		\expe{2^{ \sum_{i=1}^{N-1} f_i(\mbf{Z}_1, \ldots, \mbf{Z}_i) } }.
	\end{align*}
Continuing in the same way yields
	\begin{align*}
	\prob{ \frac{1}{N} \sum_{i=1}^{N} f_i(\mbf{Z}_1, \ldots, \mbf{Z}_i) > t }
	\le \exp\left(-N t (\log 2) + \sum_{i=1}^{N} a_i \right).
	\end{align*}
\end{proof}
The second technical lemma~\citep[Lemma 2]{CsiszarN:91gavc} is a basic result about the distribution of inner product between a randomly chosen unit vector and any other fixed vector.  It is a consequence of a result of Shannon~\citep{Shannon:59gaussian} on the distribution of the angle between a uniformly distributed unit vector and a fixed unit vector.

\begin{lemma}[Lemma 2 of~\cite{CsiszarN:91gavc}]
\label{lem:expangle}
Let $\mbf{U}$ be uniformly distributed on the unit sphere $\Sp^{d-1}$ in $\R^d$.  Then for every unit vector $\mbf{u}$ on this sphere and any $\phi \in [ (2 \pi d)^{-1/2}, 1)$, the following inequality holds:
	\begin{align*}
	\prob{ \ip{\mbf{U}}{\mbf{u}} \ge \phi } \le (1 - \phi^2)^{(d-1)/2}.
	\end{align*}
\end{lemma}

\begin{lemma}[Packing set on the unit sphere]
\label{lem:pack}
Let the dimension $d$ and parameter $\phi \in [ (2 \pi d)^{-1/2}, 1)$ be given.  For $N$ and $t$ satisfying
	\begin{align}
	-N t (\log 2) + N (N-1) (1 - \phi^2)^{(d-1)/2} < 0
	\label{eq:genpacksize}
	\end{align}
there exists a set of $K = \lfloor (1 - t) N \rfloor$ unit vectors $\mc{C}$ such that for all distinct pairs $\mu,\nu \in \mc{C}$,
	\begin{align}
	\left| \ip{\mu}{\nu} \right| < \phi.
	\label{eq:genpackip}
	\end{align}
\end{lemma}

\begin{proof}
The goal is to generate a set of $N$ unit vectors on the surface of the sphere $\Sp^{d-1}$ such that they have large pairwise distances or, equivalently, small pairwise inner products.  To that end, define i.i.d. random variables $\mbf{Z}_1, \mbf{Z}_2, \ldots, \mbf{Z}_N$ uniformly distributed on $\Sp^{d-1}$ and functions
	\begin{align*}
	f_i(\mbf{Z}_1, \ldots, \mbf{Z}_i) = \mbf{1}\left( \left| \ip{\mbf{Z}_i}{\mbf{Z}_j} \right| > \phi,\
	 j < i \right).
	\end{align*}
That is, $f_i = 1$ if $\mbf{Z}_i$ has large inner product with any $\mbf{Z}_j$ for $j < i$.  The
conditional expectation, by a union bound and Lemma \ref{lem:expangle}, is
	\begin{align*}
	\expe{ f_i(\mbf{Z}_1, \ldots, \mbf{Z}_i) | \mbf{Z}_1, \ldots, \mbf{Z}_{i-1} }
	\le 2 (i-1) (1 - \phi^2)^{(d-1)/2}.
	\end{align*}
Let $a_i = 2 (i-1) (1 - \phi^2)^{(d-1)/2}$.  Then
	\begin{align*}
	\sum_{i=1}^{N} a_i = N(N-1) (1 - \phi^2)^{(d-1)/2}.
	\end{align*}
Then Lemma \ref{lem:conc2} shows
	\begin{align*}
	\prob{ \frac{1}{N} \sum_{i=1}^{N} f_i(\mbf{Z}_1, \ldots, \mbf{Z}_i) > t }
	\le
	\exp\left(-N t (\log 2) + N (N-1) (1 - \phi^2)^{(d-1)/2} \right).
	\end{align*}
This inequality implies that as long as 
	\begin{align*}
	-N t (\log 2) + N (N-1) (1 - \phi^2)^{(d-1)/2} < 0,
	\end{align*}
then there is a finite probability that $\{\mbf{Z}_i\}$ contains a subset $\{\mbf{Z}'_i\}$ of size $\lfloor (1 - t) N \rfloor$ such that $\left| \ip{ \mbf{Z}'_i }{ \mbf{Z}'_j } \right| < \phi$ for all $(i,j)$.  Therefore such a set exists.
\end{proof}

A simple setting of the parameters gives the packing in Lemma \ref{lem:simplepack}.

\begin{proof}
Applying Lemma \ref{lem:pack} yields a set of $K$ vectors $\mc{C}$ satisfying \eqref{eq:genpacksize} and \eqref{eq:genpackip}.  To get a simple bound that's easy to work with, we can set 
	\begin{align*}
	- N t (\log 2) + N (N-1) (1 - \phi^2)^{(d-1)/2} = 0,
	\end{align*}
and find an $N$ close to this.  Setting $\psi = (1 - \phi^2)^{(d-1)/2}$, and solving for $N$ we see
	\begin{align*}
	N = 1 + \frac{t \log 2}{\psi} > \frac{t}{2 \psi}.
	\end{align*}
Now setting $K = \frac{t (1 - t)}{2 \psi}$ and $t = 1/2$ gives \eqref{eq:simplepack}.  So there exists a set of $K$ vectors on $\Sp^{d-1}$ whose pairwise inner products are smaller than $\phi$.
\end{proof}

The maximum set of points that can be selected on a sphere of dimension $d$ such that their pairwise inner products are bounded by $\phi$ is an open question.  These sets are sometimes referred to as spherical codes~\citep{splag} because they correspond to a set of signaling points of dimension $d$ that can be perfectly decoded over a channel with bounded noise.  The bounds here are from a probabilistic construction and can be tightened for smaller $d$.  However, in terms of scaling with $d$ this construction is essentially optimal~\citep{Shannon:59gaussian}.

\bibliographystyle{plainnat}
\bibliography{../JMLRfinal/ppcaCamera/privacySub}

\end{document}